\documentclass{article}

\usepackage{microtype}
\usepackage{graphicx}
\usepackage{subfigure}
\usepackage{booktabs} %

\usepackage[accepted]{icml2023}

\usepackage{amsmath}
\usepackage{amssymb}
\usepackage{mathtools}
\usepackage{amsthm}
\usepackage{xspace}
\usepackage{enumitem}
\usepackage{hyperref}
\usepackage[capitalize,noabbrev]{cleveref}
\definecolor{mydarkblue}{rgb}{0,0.08,0.45}
\hypersetup{
    pdftitle={Learning in POMDPs is Sample-Efficient with Hindsight Observability},
    pdfsubject={Proceedings of the International Conference on Machine Learning 2023},
    pdfkeywords={},
    pdfborder=0 0 0,
    pdfpagemode=UseNone,
    colorlinks=true,
    linkcolor=mydarkblue,
    citecolor=mydarkblue,
    filecolor=mydarkblue,
    urlcolor=mydarkblue,
}
\usepackage[toc, page, header]{appendix}
\usepackage{url}
\usepackage{graphicx}
\usepackage[font=scriptsize]{caption}
\usepackage{amsmath}

\usepackage{amsfonts}
\usepackage{graphicx} %

\usepackage{xcolor}
\usepackage{pdflscape}

\usepackage{algorithm,algorithmic}
\usepackage{tabularx} %
\usepackage{amsthm}
\usepackage{amsfonts}
\usepackage{amsmath}
\usepackage{amssymb}
\usepackage{mathrsfs}
\usepackage{multirow}
\usepackage{bbm}
\usepackage{nicefrac}
\usepackage[normalem]{ulem}
\usepackage{thm-restate}

\renewcommand{\subset}{\subseteq}

\newcommand{\red}[1]{\textcolor{red}{#1}}
\newcommand{\blue}[1]{\textcolor{blue}{#1}}

\DeclareMathOperator*{\argmax}{arg\,max}

\DeclareMathOperator*{\unif}{Unif}

\DeclareMathOperator*{\var}{var}

\usepackage{enumitem}

\newcommand{\beq}{\begin{equation}}

\newcommand{\eeq}{\end{equation}}
\newcommand{\beqs}{\begin{equation*}}
\newcommand{\eeqs}{\end{equation*}}

\newcommand{\pistar}{\pi^\star}

\renewcommand{\AA}{\mathcal{A}}
\newcommand{\AF}{\mathfrak{A}}

\newcommand{\OO}{\mathcal{O}}

\newcommand{\BB}{\mathcal{B}}

\newcommand{\LL}{\mathcal{L}}

\newcommand{\R}{\mathbb{R}}

\renewcommand{\sl}{\textsf{SL}}

\newcommand{\EE}{\mathcal{E}}

\newcommand{\UU}{\mathcal{U}}

\newcommand{\MM}{\mathcal{M}}
\newcommand{\N}{\mathbb{N}}

\newcommand{\regret}{\operatorname{Reg}}

\newcommand{\E}{\mathbb{E}}

\newcommand{\XX}{\mathcal{X}}
\newcommand{\YY}{\mathcal{Y}}

\newcommand{\x}{\mathbf{x}}

\newcommand{\TT}{\mathcal{T}}

\newcommand{\1}{\mathbf{1}}
\newcommand{\<}{\left<}
\renewcommand{\>}{\right>}

\newcommand{\abs}[1]{\ensuremath{| #1 |}}
\newcommand{\floor}[1]{\lfloor #1 \rfloor}
\newcommand{\ceil}[1]{\lceil #1 \rceil}

\newcommand*\diff{\mathop{}\!\mathrm{d}}

\newcommand{\pop}{\textsf{POP}}
\newcommand{\yup}{y_{\text{up}}}
\newcommand{\ydown}{y_{\text{down}}}

\definecolor{officegreen}{rgb}{0.0, 0.5, 0.0}

\newcommand{\constantCT}{4}
\newcommand{\constantCO}{8}
\newcommand{\constantCTnew}{8}
\newcommand{\constantCOnew}{8}

\newcommand{\taub}{\pmb{\tau}}

\newcommand{\epsilontildek}{\min \left\{  2,  {  2 c   X \log (X^2 A K H /\delta ) \over n_k(x, a) }\right\}}

\newcommand{\defineHighProbT}{\left\{   \forall k, h, x, a, \tau_h, \ \sum_{y', x'} O_\star(y' | x') \left( T_\star (x' | x, a) - \hat T_k(x' | x, a) \right) \alpha^{\pi_*}_{h + 1, \tau_h'}(x') \leq   \sqrt{C_T   H^3 \log (YXA HK  /\delta ) \over n_k(x, a) } \right\}}

\newcommand{\defineHighProbTab}{\left\{ \forall k \in [K], x \in \XX, \ \| O_\star(\cdot | x) - \hat O_k(\cdot | x)  \|_1 \leq \sqrt{ C_O Y \log ( Y X K H /\delta ) \over n_k(x)}  \right\}}

\newcommand{\defineHighProbC}{\left\{   \forall k, x, a, x', \  \hat T_k(x' | x, a) - T_\star(x' | x, a) = { T_\star(x' | x, a) \over 2 c } + { 2c \log(X^2A KH /\delta ) \over n_k(x, a)}    \right\}}

\newcommand{\precConst}{C}
\newcommand{\precConstValue}{21}

\newcommand{\setshort}{HOMDP\xspace}
\newcommand{\setshorts}{HOMDPs\xspace}

\newcommand{\setting}{Hindsight Observable Markov Decision Process\xspace}
\newcommand{\fnalg}{HOP-V\xspace}
\newcommand{\fnalglong}{Hindsight OPtimism with Version spaces\xspace}
\newcommand{\tabalg}{HOP-B\xspace}
\newcommand{\tabalglong}{Hindsight OPtimism with Bonus\xspace}

\newcommand{\algcomment}[1]{\textcolor{blue}{\footnotesize{\texttt{\textbf{//
          #1}}}}}

\newcommand{\otil}{\widetilde{\OO}}
\newcommand{\scalemath}[2]{\scalebox{#1}{$\displaystyle #2$}}

\theoremstyle{plain}
\newtheorem{theorem}{Theorem}[section]
\newtheorem{proposition}[theorem]{Proposition}
\newtheorem{lemma}[theorem]{Lemma}
\newtheorem{corollary}[theorem]{Corollary}
\theoremstyle{definition}
\newtheorem{definition}[theorem]{Definition}
\newtheorem{assumption}[theorem]{Assumption}
\theoremstyle{remark}

\icmltitlerunning{Learning in POMDPs is Sample-Efficient with Hindsight Observability}

\begin{document}

\twocolumn[
\icmltitle{Learning in POMDPs is Sample-Efficient with Hindsight Observability}

\icmlsetsymbol{equal}{*}

\begin{icmlauthorlist}
\icmlauthor{Jonathan N. Lee}{stan,goog}
\icmlauthor{Alekh Agarwal}{goog}
\icmlauthor{Christoph Dann}{goog}
\icmlauthor{Tong Zhang}{goog,hkust}
\end{icmlauthorlist}

\icmlaffiliation{stan}{Stanford University}
\icmlaffiliation{goog}{Google Research}
\icmlaffiliation{hkust}{HKUST}

\icmlcorrespondingauthor{Jonathan N. Lee}{jnl@stanford.edu}

\icmlkeywords{Machine Learning, ICML}

\vskip 0.3in
]
\theoremstyle{definition}
\newtheorem{example}{Example}

\printAffiliationsAndNotice{}  %

\begin{abstract}
POMDPs capture a broad class of decision making problems, but hardness results suggest that learning is intractable even in simple settings due to the inherent partial observability.
However, in many realistic problems, more information is either revealed or can be computed during some point of the learning process.
Motivated by diverse applications ranging from robotics to data center scheduling, we formulate a \setting (\setshort) as a POMDP where the latent states are revealed to the learner in hindsight and only during training.  
We introduce new algorithms for the tabular and function approximation settings that are provably sample-efficient with hindsight observability, even in POMDPs that would otherwise be statistically intractable. We give a lower bound showing that the tabular algorithm is optimal in its dependence on latent state and observation cardinalities.
\end{abstract}

\section{Introduction}
\label{sec:intro}

Sequential decision making settings where the learning agent only receives incomplete observations of its environmental state are typical in diverse practical scenarios, such as control of physical systems~\citep{thrun2000probabilistic}, dialogue and recommendation systems~\citep{young2013pomdp,shani2005mdp}, and decision making in educational or clinical settings~\citep{ayer2012or}. Typically studied within the framework of a Partially Observable Markov Decision Process (POMDP), classical literature on such problems provides hardness results on sample and computationally efficient learning, even in simple settings with small action, observation and state spaces, in stark contrast to the MDP setting where the state is fully observable. Fueled by this gap, there is a body of literature that characterizes observability conditions when the sequence of observations reveals enough information about the latent states to permit sample-efficient learning.
In this paper, we ask if the motivating practical applications sometimes allow a more informative sensing of the underlying state for some part of the learning process. We formulate a novel learning setting called a \setting (\setshort), and provide learning algorithms that are significantly more sample-efficient than 
those for general POMDPs.

For motivation, let us consider robotic control, where we want our robot to sense its state using a relatively cheap camera sensor upon deployment. However, during training, it is common to allow a more expensive sensing of the state, using simulators, higher-fidelity cameras, lidars, or even full-fledged motion capture setups~\citep{pinto2017asymmetric,pan2017agile,chen2020learning}. 
In a completely different style of scenarios, \citet{sinclair2022hindsight} discuss problems such as scheduling in a data center, where the unknown lifetime of a job creates partial observability of the state when the job is scheduled. This partial observability is resolved when the job actually concludes. While the two examples are very different, they share a similarity. The learner needs a decision making policy to act based on partial observations alone, due either to resource/sensor constraints upon deployment or to fundamental lack of information at the time of decision. However, the underlying state eventually gets revealed, either intrinsically, or due to extrinsic measurements during the training process. We refer to this eventual observation of the latent environment state as \emph{hindsight observability}, 
and study learning settings where the learner acts based on partial state observations, but observes the true latent states eventually upon the conclusion of the trajectory.

We start by noting that learning in a \setshort remains considerably challenging in comparison with MDPs, as the learner's policy needs to depend on observations during deployment (e.g. robotics, scheduling) and sometimes even during training (e.g. scheduling).
Hence we cannot use MDP learning techniques directly. At the same time, the \setshort model eliminates the identifiability or observability conditions that are crucial to success in POMDP learning, since the hindsight observation of the latent state allows us to associate latent states and corresponding observations, albeit with a delay. This makes the \setshort an intermediate step between the complexity of MDPs and POMDPs, which is practically prevalent 
as our earlier examples suggest.%

\paragraph{Our contributions} In addition to formalizing the \setshort framework, and showing its broad applicability across diverse settings, such as sim-to-real robotics, high-frequency control, meta-learning and scheduling problems in Appendix~\ref{app:more_applications}, we make the following key contributions:

\begin{enumerate}[nosep,leftmargin=12pt]
\item When the latent states and observations are both finite, we provide an algorithm \tabalg, which finds an $\epsilon$-optimal policy using at most $\otil\left(\frac{XYH^5 + XAH^4}{\epsilon^2}\right)$ trajectories, where the \setshort contains $X$ latent states, $Y$ observations, $A$ actions, and the horizon is $H$. In contrast with standard POMDP learning results, there is no observability-related parameter in our bound.
\item We show an $\Omega(\frac{XY}{\epsilon^2})$ lower bound, meaning that \tabalg scales optimally with latent states and observations.
\item We develop a general algorithm, \fnalg, which allows function approximation for both latent states and observations, and allows representation learning in the latent state space. The sample complexity of \fnalg depends on the statistical complexity of function classes used to learn latent state transitions and emissions, along with a rank parameter of the latent state transitions.  Again, there are no observability conditions in contrast with standard POMDP results. %
\end{enumerate}

\section{Related Work}

There has been significant progress in understanding the sample efficiency of reinforcement learning in the fully observable setting of MDPs. For tabular MDPs (finite states and actions), upper and lower bounds for sample complexity and regret are well known~\citep{auer2008near,dann2015sample,osband2016lower,azar2017minimax,dann2019policy, zanette2019tighter}. Similar results have been established for MDPs that satisfy certain structural conditions, enabling function approximation~\citep{jiang2017contextual, sun2019model,jin2020provably,agarwal2020flambe,du2021bilinear,jin2021bellman,foster2021statistical,agarwal2022model}.

Relative to MDPs, the sample complexity of reinforcement learning in POMDPs is less understood. Classical hardness results suggest learning in POMDPs can be both computationally and statistical intractable even for simple settings~\citep{krishnamurthy2016pac}. This hardness has spurred researchers to identify conditions under which sample efficient learning is still possible in POMDPs. Block MDPs~\citep{krishnamurthy2016pac,du2019provably} and decodable MDPs~\citep{efroni2022provable} are special classes of POMDPs in which the current observation (or last few observations) can exactly decode the current latent state. Several works study more general observability conditions beyond decodability~\citep{azizzadenesheli2016reinforcement,guo2016pac,jin2020sample,golowich2022planning,liu2022partially,liu2022optimistic,uehara2022computationally}. %
Sample complexity bounds under these conditions often depend crucially on parameters that quantify the degree of observability.
\citet{liu2022optimistic,zhan2022pac,zhong2022posterior} provide similar conditions for general predictive state representations (PSRs).
Although aimed at the same objective of learning policies for partially observable settings, our work uses hindsight observability to circumvent any additional parameters or assumptions on the emission function.

Empirically, a number of works successfully leverage latent state information during training to improve sample efficiency. \citet{pinto2017asymmetric,baisero2021unbiased} study \textit{asymmetric} actor-critic algorithms where the critic uses the latent state while the actor uses  observations, allowing the learned policy to later interact with only observations. \citet{pan2017agile,chen2020learning,warrington2021robust} use distillation-based approaches where they train an expert policy on latent states and then later imitate it with an observation-based policy. Similar settings also appear as privileged information~\citep{kamienny2020privileged} or resource-constrained RL~\citep{regatti2021offline}.
However, these prior works do not address sample complexity and exploration.

 Motivated by resource allocation, \citet{sinclair2022hindsight} study a similar hindsight problem, where the unobserved part of the latent state is not affected by the learner's actions, and dynamics are fully known in hindsight. Hence, they study a planning problem in hindsight with no need for exploration, unlike the general \setshort setting considered here.

\citet{kwon2021rl,zhou2022horizon} study a latent MDP model, where the latent state contains an additional identifier of the active MDP for each episode, and the identifier is revealed in hindsight during training. Our setting is significantly more general, but shares similar motivation. We compare our bounds in Section~\ref{sec::finite}.

\section{\setting}\label{sec::homdp}
The underlying model of the \setshort setting is the same as a POMDP; the difference lies in what information is revealed to the learner and when. We first review the relevant quantities of a POMDP and subsequently introduce the hindsight observability and learning protocol in a \setshort.

\subsection{Preliminaries}
For $n \in \N$, we use $[n]$ to denote the set $\{1, 2, \ldots, n\}$. For a set $S$, $\Delta(S)$ denotes the set of (appropriately defined) densities over $S$.  For $h \in \N$, we use $a_{1:h}$ to denote $(a_{1}, \ldots, a_h)$.

 We consider an episodic partially observable Markov decision process (POMDP) $\MM$ with episode length $H$, latent state space $\XX$, observation space $\YY$, and action space $\AA$. When these are finite, we denote their respective cardinalities as $X := \abs{\XX}$, $Y := \abs{\YY}$, and $A := \abs{\AA}$. An initial latent state $x_1$ is sampled from a fixed and known initial state distribution $\rho \in \Delta(\XX)$. The process evolves according to transition function $T_\star: \XX\times \AA \to \Delta(\XX)$, acting on the latent states. 
 When the learner visits a latent state $x$, the environment generates an observation $y \in \YY$ according to the conditional emission function $O_\star : \XX\to \Delta(\YY)$. In particular, in each episode, a (latent) trajectory $\bar \tau = (x_1, y_1, a_1, \ldots, x_H, y_H, a_H, x_{H + 1})$ is generated where $x_1 \sim \rho(\cdot)$, $x_{ h + 1} \sim T_\star(\cdot | x_h, a_h)$, $y_h \sim O_\star(\cdot | x_h)$, and  the learner selects $a_{1:H}$. We include $x_{H + 1}$ (from taking $a_H$ in $x_H$) as a latent variable for convenience. When referring to an (observed) trajectory of just observations and actions, we use $\tau = (y_1, a_1, \ldots, y_H, a_H)$. 
 We assume there is a known deterministic reward function $r: \XX \times \AA \to [0, 1]$. Our results can be generalized readily to stochastic, observation-dependent rewards. 

As is standard in POMDPs, we  consider the setting where the learner interacts with the environment by specifying a \textit{history-dependent} policy $\pi  \ : \ (\YY \times \AA)^* \times \YY \to \Delta(\AA)$ which takes as input a (variable) $h$-length history of observations $y_{1:h}$ and $(h - 1)$-length history of actions $a_{1: h - 1}$ and outputs a distribution over actions. That is, the learner's policy does not get to observe any of the latent states $x_{1:H+1}$ during execution of $\pi$.  For conciseness, we  denote the partial histories as $\tau_h := (y_{1:h}, a_{1:h - 1})$ and $\bar \tau_h := (y_{1:h}, x_{1:h}, a_{1:h - 1})$, which includes the observation $y_h$ and latent state $x_h$ (if applicable) at step $h$.  

For a policy $\pi$, we denote the expected cumulative reward over an episode by
\begin{align}\label{eq::policy-value}
	v(\pi) = \E_{\pi} \left[  \sum_{h \in [H]} r(x_h, a_h) \right],
\end{align} 
where $\E_\pi$ is the expectation taken over trajectories in the POMDP under policy $\pi$.

\subsection{ Hindsight observability }

\begin{figure}
    \centering
    \includegraphics[width=3in]{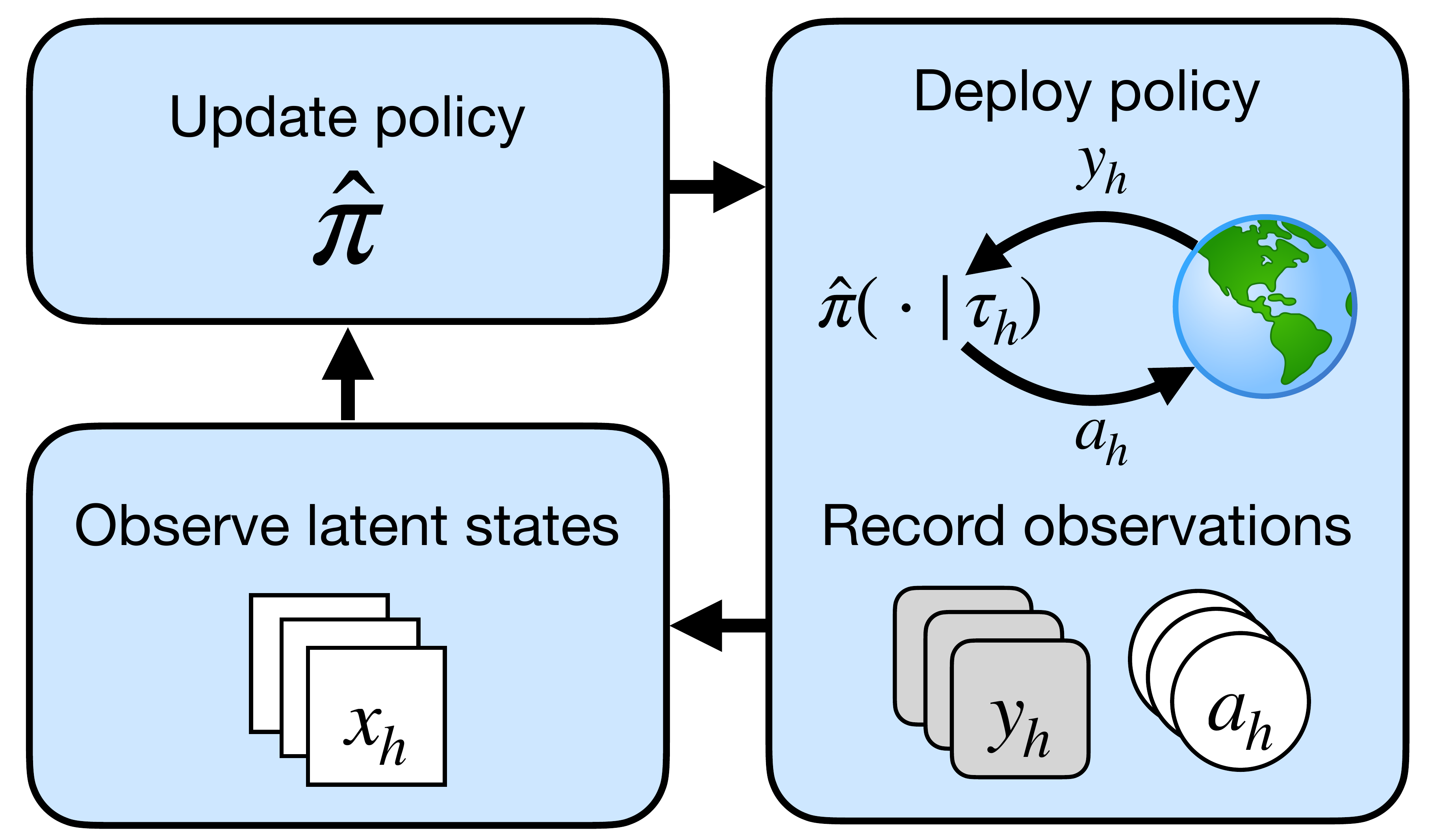}
    \caption{A \setshort model at train time. The learned history-dependent policy $\hat \pi$ is deployed and takes actions $a_{1:H}$ using only observations $y_{1:H}$. After deployment, the environment reveals the latent states $x_{1:H+1}$. The policy updates with both the latent states $x_{1:H + 1}$ and observations $(y_{1:H}, a_{1:H})$. At test time, only a history-dependent policy is deployed.
    }
    \label{fig::problem}
\end{figure}
 Now we formally introduce the \setshort setting and describe the interaction protocol, i.e., how the learner interacts with the environment and receives information. We also illustrate this description in the accompanying Figure~\ref{fig::problem}. Along the way, we highlight differences with the standard POMDP and MDP settings. 
There are two phases in the \setshort: train time and test time.

During train time, the learner interacts with the environment over $K \in \N$ rounds (episodes).
At any given round $k \in [K]$, the learner produces a history-dependent policy $\hat \pi_k$ which is deployed in the partially observable environment as if the learner is interacting with a standard POMDP. During execution of the episode $k$ at time $h \in [H]$, the environment reveals only the current observation $y_h$ to the learner. The policy can thus base its decision $a_{h}$ on only the partial history $\tau_{h} = (y_{1:h}, a_{1:h -1})$ of interactions in that episode. Once the $k$th episode is completed, the latent states $x_{1:H+1}$ are revealed to the learner in hindsight, hence the terminology \textit{hindsight observability}. The learner can then generate a new policy $\hat \pi_{k + 1}$ using information from $(x_{1:H+1}, y_{1:H}, a_{1:H})$ as well as that of all previous episodes. This is the key difference between \setshorts and standard POMDPs where the latent states are never revealed and the learner generates the policy from only previous observations, actions, and rewards alone. In MDPs, on the other hand, the latent state is observed instantaneously and the policy can directly map a latent state to an action.

The train time phase may be followed by a test time phase where a single history-dependent policy $\hat \pi$ is deployed but the learner does not observe latent states or update the policy after committing to $\hat \pi$. To determine $\hat \pi$, the learner can use all of the information collected over the $K$ episodes at train time, including the latent states observed in hindsight. The quantity $v(\hat \pi)$ evaluates the quality of $\hat \pi$.
We let $\pistar$ denote the optimal observation-based policy maximizing $v(\pistar)$ and measure the sub-optimality of the learner's policy $\hat \pi$ by the difference $v(\pistar) - v(\hat \pi)$. Again, in contrast with an  MDP, a \setshort never reveals the latent state at test time.

We are primarily concerned with PAC sample complexity bounds controlling the suboptimality of $\hat \pi$, but some algorithms also address the regret problem at train time where the regret for all $K$ rounds is measured as
\begin{align*}
	\regret(K) = \sum_{k \in [K]} v(\pistar) - v(\hat \pi_k).
\end{align*}

\begin{example}[Sim-to-real robotics]
 In \textit{sim-to-real robotics} \citep{pinto2017asymmetric}, one trains an image-based policy in a simulator with access to the underlying states. The goal is to deploy the image-based policy in the real world. $\XX$ is the set of robot and object positions and poses which are observable in the simulator during training.  $\YY$ is the set of image observations from a camera, which is the only modality available at test time. $\XX$ and $\YY$ are both continuous and high dimensional. $\AA$  is the set of control inputs the robot can take such as joint torques. Note that the latent states are available without delay at train time in this example. However, we later show that this variant of the problem interestingly does not yield significant statistical advantage in theory because the desired policy at test time is still history-dependent (see the discussion following Theorem~\ref{thm::lower-bound}). Empirically, history-dependent policies are still preferred to state-based policies even during train time despite access to the state to facilitate better sim-to-real transfer.
 \label{ex:sim-to-real}
 \end{example}
 
 \begin{example}[Data center scheduling]
 In \textit{data center scheduling}~\citep{sinclair2022hindsight}, described in the introduction, $\YY$ is the observable state of the submitted, processing, and completed jobs as well as their allocations to servers, which is available at the time of decision-making. $\XX$ is a concatenation of $\YY$ with lifetime lengths of the submitted/processing jobs and the arrival times of future jobs. This information is available, but only in hindsight. Depending on the setup, $\XX$ and $\YY$ can be relatively succinct here. $\AA$ is the set of allocation actions for currently submitted jobs. 
\end{example}

\subsection{Comparison with hardness of learning in POMDPs} 
\label{sec:pomdp-comp}
Both POMDPs and \setshorts share the use of history-dependent policies during execution at test time. However, a POMDP never reveals the association between observations and latent states, leading to a lack of identifiability and exponential in $H$ lower bounds even for simple ones \citep{krishnamurthy2016pac}.
As discussed previously, numerous recent papers \citep{liu2022partially,jin2020sample,cai2022reinforcement,liu2022optimistic,zhan2022pac} investigate observability conditions under which sample-efficient learning is possible by ensuring $O_\star$ reveals enough about the distribution of possible latent states. This yields sample complexity bounds that incur an unavoidable dependence on the minimum singular value of $O_\star$~\citep{liu2022partially}, or related parameters. 

However, settings where such observability conditions are satisfied may still preclude many practically interesting partially observable problems. Our objective in this paper is to understand to what extent the addition of hindsight observability in a \setshort can make learning in partially observable settings sample-efficient without relying on observability parameters.

\section{Learning in Finite \setshorts}\label{sec::finite}

We now turn to the design of efficient algorithms for learning in the \setshort model. We begin with the setting where the latent state space $\XX$ and observation space $\YY$ have finite cardinalities $X := |\XX|$ and $Y := |\YY|$. We introduce a new algorithm, \tabalg, which naturally extends minimax optimal results (in $X$ and $A$) for tabular MDPs to the \setshort model. Our proposed algorithm is model-based, leveraging the intuition that, provided with the latent states $x_{1:H+1}$, one should be able to estimate the transition and emission functions, $T_\star$ and $O_\star$. We start with the algorithm, before presenting the sample complexity guarantee.

\subsection{The \tabalg algorithm}

Our algorithm, \tabalglong (\tabalg) estimates the transition and emission models, and subsequently finds an optimal policy in this learned model using a reward bonus to encourage exploration. The design of bonus is a key novelty in \tabalg, relative to its MDP counterparts, as we will discuss shortly.

Before describing the algorithm, we define a planning oracle which is used in the algorithm to compute the exploration policies. Note that planning in a \setshort is identical to a POMDP, as we seek an optimal history-dependent policy. 

\begin{definition}[Optimal planner]
The POMDP planner \pop{} takes as input a transition function $T$, an emission function $O$, and a reward function $r$ and returns a policy $\pi = \pop(T, O, r)$ such that $v(\pi) = \max_{\pi'}v_{\MM(T, O, r)}(\pi')$, where $\MM(T, O, r)$ denotes the POMDP model with latent transitions $T$, emissions $O$, and reward function $r$.
\end{definition}

While it is known that planning in POMDPs is PSPACE-hard in general \citep{papadimitriou1987complexity}, there are many special classes of POMDPs for which planning is computationally efficient. Alternatively, it is possible in practice to use one of many existing approximate POMDP planners; however, this will likely weaken the subsequent theoretical guarantees of this section up to some approximation error. Regardless, this is much milder computational assumption than what is sometimes made in comparable POMDP literature~\citep{jin2020sample}.

\begin{algorithm}
\caption{\tabalglong (\tabalg)} 
\label{alg:tabular}
\begin{algorithmic}[1]
	
	\STATE \textbf{Input}: POMDP planner $\pop$.
	\STATE Initialize emission and transition models $\hat O_1, \hat T_1$.%
	\STATE Initialize  $n_1(x) = n_1(x, a) = 0$ for all $x \in \XX, a \in \AA$.
	\STATE Set bonus parameters\label{line:betas}\\
	        \mbox{$\beta_1 = \constantCT H^3 \log ( Y X A H K / \delta), \beta_2 = \constantCO Y \log ( Y X K H /\delta ).$}
			
			\FOR{ $k = 1, \ldots, K$ }
			
            \STATE \algcomment{Set reward bonuses}			

			\STATE $\epsilon_k(x, a) =  \min \left\{  \sqrt{ \beta_1   \over n_k(x, a)  },  2 H  \right\}$ \label{line::bonus-trans}
			
			\STATE  $\epsilon_k(x) = \min\left\{  \sqrt{ \beta_2   \over n_k(x)  },  2 \right\}$ \label{line::bonus-obs}
			
			\STATE $\hat r_{k}(x, a) = r(x, a) +  H\epsilon_k(x) + \epsilon_k(x, a)$
			
			\STATE \algcomment{Plan, deploy hist.-dependent policy}
			\STATE $\hat \pi_k = \pop( \hat T_k, \hat O_k, \hat r_k  )$
			
			\STATE Run $\hat \pi_k$ and observe trajectory $\tau^k = (y_{1:H}^k, a_{1:H}^k)$. 
			
			\STATE \algcomment{Hindsight observation}
			
			\STATE Observe latent states $x^k_{1:H+1} = (x^k_1, \dots, x^k_{H+1})$.

            \STATE \algcomment{Update models}

		    \STATE $n_{k + 1}(x) = \sum_{\ell \in [k], h \in [H]} \1\{ x^{\ell}_h = x \}$.
			
			\STATE $n_{k + 1}(x, a) = \sum_{\ell \in [k], h \in [H]} \1 \{ x^\ell_h = x \wedge a^\ell_h =a\} $.

			\STATE Update $\hat T_{k + 1}$ via \eqref{eq::tabular-transition-update}
			
			\STATE Update $\hat O_{k + 1}$ via \eqref{eq::tabular-emission-update}

			\ENDFOR		
		\end{algorithmic}
		
	\end{algorithm}

\tabalg operates over $K$ rounds, starting with arbitrary guesses $\hat T_1$ and $\hat O_1$ of the model. At round $k$, it computes reward bonuses based on the uncertainty in the estimates $\hat T_k$ and $\hat O_k$, which is quantified by the number of visits to each latent state $x$ and latent-state action pair $(x, a)$ from the dataset. We define these bonuses in $\epsilon_k(x)$ and $\epsilon_k(x, a)$ in lines~\ref{line::bonus-trans} and~\ref{line::bonus-obs} with parameters given in line~\ref{line:betas}. Note that the $\epsilon_k(x)$ bonus is in addition to the typical bonus in MDPs. Informally $\epsilon_k(x,a)$ captures our uncertainty in the estimation of $T_\star$, while $\epsilon_k(x)$ measures it for $O_\star$. For instance, even if we know $T_\star$, we need to visit each latent state to estimate its emission process for the subsequent planning, capturing the need for the additional $\epsilon_k(x)$ bonus.

We construct a reward function $\hat r_k$ by adding the bonuses to $r$. We then invoke the planner \pop{} using $\hat T_k$, $\hat O_k$, and the reward function $\hat r_k$ to generate an optimistic history-dependent policy $\hat \pi_k$. As we show in the proof, the estimated value of $\hat \pi_k$ under the current model over-estimates the true value of $\pistar$ with high probability.
We then deploy the optimistic policy $\hat \pi_k$ in the environment to generate a trajectory of observations $y_{1:H}$ and actions $a_{1:H}$. We further observe the latent states $x_{1:H+1}$ in hindsight. Finally, using the new information from the trajectory and in hindsight, we update the models with empirical estimates using all the past data:
\begin{align}\label{eq::tabular-transition-update}
	&\scalemath{0.95}{\hat T_{k + 1}(x' | x, a) = \hspace{-0.4cm}\sum_{\ell \in [k], h\in [H]}\hspace{-0.4cm}{\1 \left\{ x^\ell_h = x, y^\ell_h=y, x^\ell_{h+1}=x' \right\}\over n_{k + 1} (x, a)}} \\
	&\hspace{3mm}\hat O_{k + 1} (y | x) =    \sum_{\ell \in [k], h\in [H]}{ \1 \left\{  x^\ell_h = x, y^\ell_h = y \right\} \over n_{k + 1} (x) } \label{eq::tabular-emission-update}
\end{align}
for all $x, x', a, y$ where $n_{k + 1}(x, a)$ and $n_{k +1}(x)$ are defined in Algorithm~\ref{alg:tabular}.
Note that both the calculation of the uncertainty bonuses and the model updates are possible only due to the hindsight observability that reveals the latent states $x^\ell_{1:H+1}$ for $\ell \in [k - 1]$. In general POMDPs, such calculations are not available.

\subsection{Regret and sample complexity bounds}

We now present the main guarantees for \tabalg in \setshorts. While we are primarily concerned with sample complexity bounds, \tabalg readily admits a regret bound.

\begin{theorem}\label{thm::tabular} Let $\MM$ be a \setshort model with $X$ latent states and $Y$ observations.
With probability at least $1 - \delta$, \tabalg outputs a sequence of policies $\hat \pi_1, \ldots, \hat \pi_K$ such that
\begin{align*}
 \regret(K)  = \widetilde \OO \left( \sqrt{{ (XY H^5 + XAH^4) K \iota } }  \right) ,
\end{align*}
where $\iota = \log ({2X^2 Y A KH }\delta^{-1})$ and $\widetilde \OO$ omits lower-order terms in $K$.
\end{theorem}
The full bound, including lower order terms, and the proof can be found in Appendix~\ref{app::tabular}.
A standard online-to-batch conversion reveals that \tabalg{} learns an $\epsilon$-optimal policy at test time with probability at least $1 - \delta$ with sample complexity
\begin{align*}
    K = \widetilde \OO\left( { X Y H^5  + XAH^4  \over \epsilon^2 } \right),
\end{align*}
omitting log factors and lower-order terms in $\epsilon$.
We highlight several conceptual  implications of the results.
\begin{itemize}
[nosep,leftmargin=10pt]
    \item The bounds do not have dependence on any \emph{observability parameter}, which typically measures the degree to which one can decode the latent distribution from observations in POMDPs. For instance, guarantees of \citet{liu2022partially} depend polynomially on the inverse of the minimum singular value of $O_\star$ and in fact this dependence is necessary in general POMDPs~\citep[see e.g. Theorem 6 in][]{liu2022partially}. This precludes efficient learning in a wide class of partially observed problems (such as vision-based robotics applications with occlusions). By leveraging hindsight observability, our results show that it is possible to circumvent this hardness while still learning a near optimal history-dependent policy for test time. 
    
    \item The leading terms depend linearly on the number of observations $Y$ and latent states $X$. Inspecting just the dependence in $X$ and $A$, the result of Theorem~\ref{thm::tabular} can thus be viewed as a natural extension of minimax regret (and sample complexity) results for the MDPs. This is known to be unimprovable in general even for full information MDPs~\citep{dann2015sample,osband2016lower}. However, due to the added complexity of partial observability during deployment, a linear term in $Y$ is also present our bound. Our lower bound in Theorem~\ref{thm::lower-bound} shows that the linear $XY$ dependence is minimax optimal. We remark prior work on POMDPs has yielded large polynomial dependence on $X$ and  $Y$ in contrast to our linear dependence here.
    \item An interesting observation of \tabalg  is that the exploration bonus need only happen at the latent state level, rather than needing to explore histories. This suggests that little structure might be needed to learn $O_\star$.\footnote{Indeed, we show in Appendix~\ref{app::fn-simple} that we can incorporate function approximation of $O_\star$ without additional structural conditions beyond realizability as long as the latent model is tabular.}
    We will see in Section~\ref{sec::fn} that this observation becomes deeper when incorporating function approximation.
\end{itemize}

The dependence on the horizon $H$, although still polynomial, is likely suboptimal; however, we conjecture that it can be improved using known tools for MDPs. Since our focus is  on the impact of $X$ and $Y$ and understanding the fundamental efficiency gaps between \setshorts and POMDPs, we leave optimizing these additional factors for future work. As discussed before, \setshorts are a generalization of latent MDPs with identifiers labeled in hindsight studied by \citet{kwon2021rl,zhou2022horizon}. Ignoring $H$ and specializing our bound to their setting, it matches \citet{kwon2021rl}. \citet{zhou2022horizon} is better by a sparsity factor, but is more specialized and not applicable to general \setshorts.

\paragraph{Proof intuition.} \tabalg resembles typical optimistic algorithms for MDPs. However, a na\"ive analysis by computing confidence intervals on $\hat T_k - T_\star$ and $\hat O_k - O_\star$ results in an $\OO(X^2)$ scaling of sample complexity. We follow the MDP literature in constructing more careful confidence bounds on appropriate value functions instead. This requires some care as value functions are history-dependent in a \setshort. In particular, we need to control the required exploration as a function of latent state visitations, while reasoning over history-dependent value functions. Combining these ideas carefully, which we present in detail in Appendix~\ref{app::tabular}, gives the proof of Theorem~\ref{thm::tabular}. We note that most POMDP analyses do suffer from the $\OO(X^2)$ or worse scaling, as they do not reason via value functions.

\section{Limits of Learning in Tabular \setshorts}

\begin{figure}
    \centering
    \includegraphics[width=2.5in]{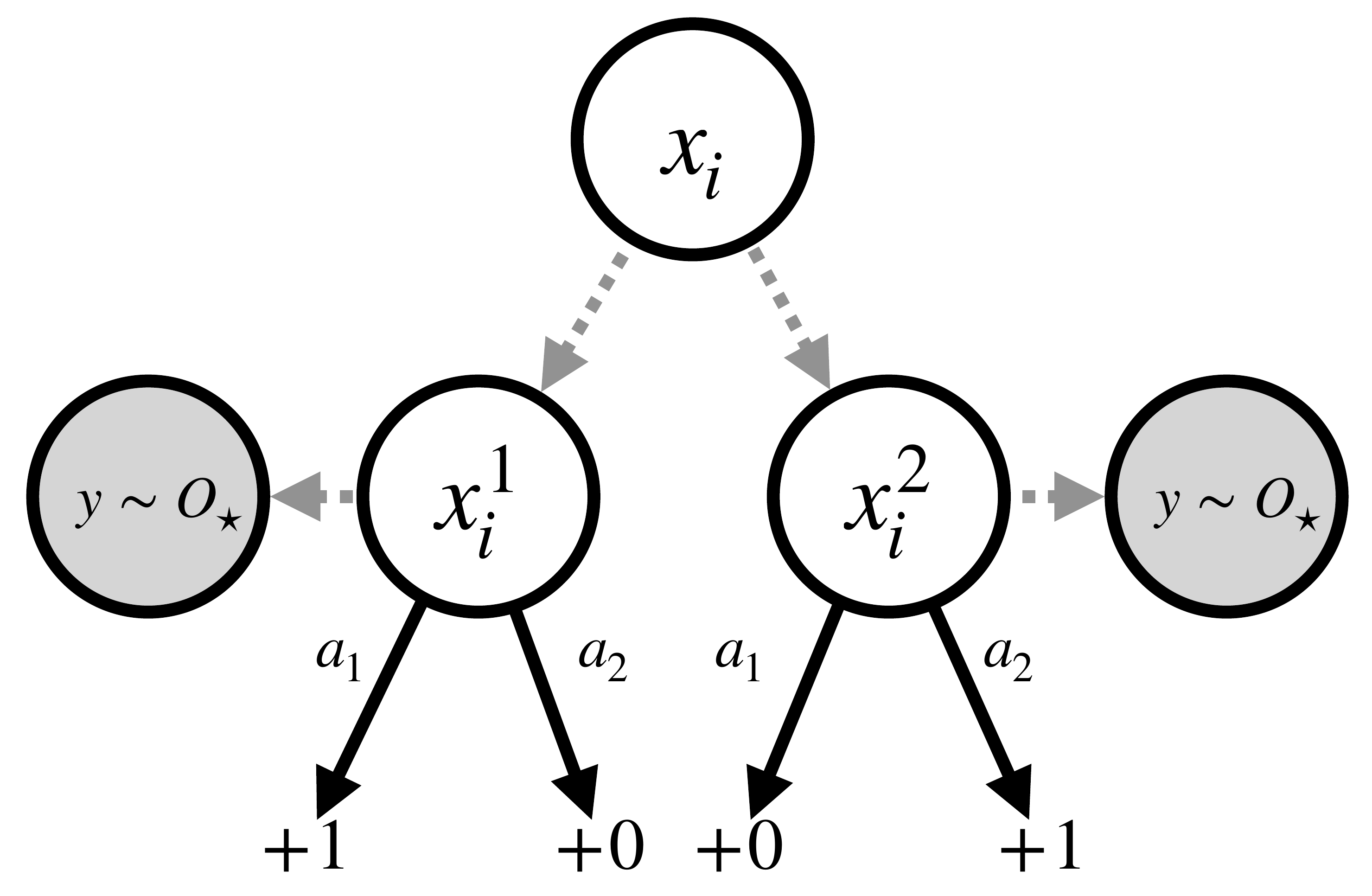}
    \caption{Simplified case of the lower bound construction. The learner starts in $x_i$ and randomly transitions to $x_i^1$ or $x_i^2$. The optimal action to achieve reward $+1$ is different depending on which latent state is visited, but the observation $y$ obfuscates sensing of the latent state.}
    \label{fig::lbsmall}
\end{figure}

We now show that the upper bound of the previous section is optimal in $XY$ for the tabular setting. We present a  new information-theoretic lower bound for the  tabular \setshort setting. The proof is in Appendix~\ref{app::lb}.

\begin{restatable}{theorem}{lowerbound}
		\label{thm::lower-bound}

Fix $\epsilon \leq 1/64$ and $X, Y \in \N$ such that $Y\geq 6$, $(X + 1) \geq 128\log 2$. For any algorithm $\AF$ producing a policy $\hat \pi$ in $K$ episodes of interaction, there exists a \setshort with the aforementioned cardinalities and $H \asymp \log_2(X)$ and $A = 2$ such that $\AF$ needs 
\begin{align*}
    K = \Omega\left({XY / \epsilon^2 }\right)
\end{align*}
to guarantee $\E \left[ v(\pistar) - v(\hat \pi) \right] \leq \epsilon$,
where the expectation is taken over randomness in the data and algorithm.

\end{restatable}

The lower bound is information-theoretic, meaning that no algorithm can do better than this.
The lower bound of Theorem~\ref{thm::lower-bound} matches the  $XY$ leading term of the upper bound in Theorem~\ref{thm::tabular}, suggesting that our algorithm is minimax optimal in $X$ and $Y$ for large $K$. Recall that existing lower bounds for learning in MDPs necessitate $\Omega\left(\nicefrac{XA}{\epsilon^2}\right)$ episodes of interaction,  which accounts for the other leading term in our upper bound~\citep{dann2015sample,osband2016lower}.
Since POMDPs are more general than \setshorts, our lower bound also applies to POMDPs.

\textbf{Hindsight vs.\ foresight observability.}
Our lower bound construction does not distinguish between the latent state $x_h$ being simultaneously revealed along with $y_h$, or only in hindsight. The key bottleneck is in the construction of the history-dependent policy at test time. Therefore, revealing states simultaneously is no easier statistically than revealing them in hindsight, at least in a minimax sense. This lends credence to framing a broader class of problems such as sim-to-real robotics as \setshorts by dealing with latent states \textit{after} execution of a policy even though the state is always observable during training. Because one seeks a history-dependent policy in the end, it is just as hard statistically.

\paragraph{Intuition for lower bound construction.}
We derive the lower bound by constucting a class of hard \setshort models in the form a binary tree for the latent states, like many hard MDP constructions. At the final layers of the tree is a collection of $\Omega(X)$ subproblems, each of the form of Figure~\ref{fig::lbsmall}. In each subproblem there are $3$ latent states. The learner starts in $x_i$ and transitions randomly to either of the child states $x_i^1$ and $x_i^2$ with equal probability and must take an optimal action that depends on which one it visits. The difficulty is that $O_\star$ is biased slightly towards half of the observations depending on the latent state. As a result, in order to effectively match the value of the optimal policy, the learner must interact at least $\Omega(Y)$ times with this subproblem. 
To prove there is linear dependence on $XY$, we leverage the binary tree described earlier. 
In short, the learner transitions randomly down the tree until it reaches one of $\Omega(X)$ independent subproblems (decodable from the observations), where it must play optimally. As we remarked earlier, the bottleneck is on the history-dependent policy deployed at test-time, so observing the latent states simultaneously at training does not help in solving this construction.

\section{Generalization in HOMDPs} \label{sec::fn}

While the tabular setting considered so far is important for building intuition, and tabular latent states are particularly reasonable in many settings, several practical domains of interest necessitate continuous and high-dimensional latent states and observations. In this section, we study \setshorts where both $\XX$ and $\YY$ can be infinitely large, and we employ function approximation for sample-efficient learning. Mirroring prior work in MDPs and more recently POMDPs, this requires structural conditions on the underlying latent state transitions $T_\star$. We now describe one such  condition which has been widely studied in the MDP literature, and present an algorithm and sample complexity guarantees.

\subsection{Low-rank latent transition dynamics}
We initiate this investigation
with a well-studied model in the MDP literature: the low-rank MDP \citep{barreto2011reinforcement,jiang2017contextual,agarwal2020flambe}. We study \setshorts where the underlying latent state MDP is low-rank.

\begin{definition}
A transition function $T_\star$ admits a low-rank decomposition with rank $d$ if there exist vector functions $\phi_\star: \XX \times \AA \to \R^d$ and $\psi_\star: \XX \to \R^d$ such that
\begin{align*}
    T_\star(x' | x, a) = \phi_\star(x, a)^\top \psi_\star(x') \quad \forall x, x' \in \XX, a \in \AA
\end{align*}
Furthermore, $\phi_\star$ satisfies $\sup_{x, a} \| \phi_\star(x, a) \|_2 \leq 1$ and $\psi_\star$ satisfies $\| \int_{x'} \psi_\star(x') \diff x' \|_2 \leq \sqrt{d}$.
\end{definition}
Note that in the low-rank setting, we do not assume that the feature embedding $\phi_\star$ is known, unlike the linear MDP setting \citep{jin2020provably} which crucially leverages this knowledge. We defer details, motivations, and comparisons to the original papers on the matter \citep{agarwal2020flambe}.

We assume that the learner has access to function classes $\TT$ and $\Theta$ for approximation of $T_\star$ and $O_\star$, respectively. For simplicity of the analysis, we assume that they are finite but  large, and thus we desire a sample complexity which is logarithmic in $|\TT|$ and $|\Theta|$, with no explicit dependence on $|\XX|$ or $|\YY|$. This formulation also automatically captures the representation learning problem for the latent states~\citep{agarwal2020flambe}. 
As is standard, we assume that the function classes are proper and satisfy realizability.

\begin{assumption}
\label{asmp::realizability}
$T_\star \in \TT$ and $O_\star \in \Theta$. Furthermore, for all $T \in \TT$ and $O \in \Theta$, $T(\cdot | x, a) \in \Delta(\XX)$ and $O(\cdot | x,a) \in \Delta(\YY)$ for all $x, x' \in \XX$ and $a \in \AA$.
\end{assumption}

The above requires that all candidates in the class can form valid distributions (i.e., they are proper); this can be  satisfied by simply discarding those in $\TT$ and $\Theta$ that are improper.

\subsection{The \fnalg algorithm}

\begin{algorithm}[t]
		\caption{\!\mbox{\fnalglong\!(\fnalg)}}
		\label{alg::gfa}
		\begin{algorithmic}[1]
			
			\STATE \textbf{Input}: Transition class $\TT$, Emission class $\Theta$.
			\STATE Set $K' = \floor{K/ H}$.
			\STATE Set $\beta_\TT  = 2 \log \left( K' | \TT | / \delta \right)$.
			\STATE Set $\beta_\Theta   =2 \log \left (K' | \Theta | / \delta \right)$.
			\STATE Initialize $\TT_1 = \TT$. and $\Theta_1  = \Theta $. 
			
			\FOR{ $k = 1, \ldots, K'$ }
			
			\STATE \algcomment{Optimistic planning}
			\STATE Solve \begin{align*}\hat \pi_k, \hat O_k, \hat T_k = \argmax_{\pi \in \Pi, T \in \TT_k, O \in \Theta_k} v_{\MM(T, O) } (\pi)  \end{align*}
			\FOR{ $h \in [H]$ }
			    \STATE \algcomment{Deploy hist.-dependent policy}
    			\STATE Construct exploration policy $\tilde \pi_k = \hat \pi_k \circ_{h} \unif(\AA)$
    			
    			\STATE Deploy $\tilde \pi_k$ and observe trajectory $(y_{1:H}, a_{1:H})$.
    			
    			\algcomment{Hindsight observation}
    			\STATE Observe latent states $x_{1:H+1}$.
    			
    			\STATE Set $y^k_h := y_h$, $a^k_h := a_h$, $x^k_h  := x_h$, $\tilde x^k_{h} := x_{h + 1}$.
			
			\ENDFOR

			\STATE Update version spaces with
			\begin{align*}
			    \TT_{k + 1}  & = \left\{ T \in \TT_k \ : \ \hat \LL^1_k(T) \geq \max_{T' \in \TT_k}  \hat \LL^1_k(T') - \beta_\TT \right\} \\
			    \Theta_{k + 1} & = \left\{ O \in \Theta_k \ : \ \hat \LL^2_k(O) \geq \max_{O' \in \Theta_k} \hat \LL^2_k(O') - \beta_\Theta \right\} 
			\end{align*}

			\ENDFOR		
		\end{algorithmic}
	\end{algorithm}

We now introduce the algorithm Hindsight OPtimism with Version spaces (\fnalg) for function approximation in \setshort models. \fnalg divides the $K$ rounds into $K' = \floor{K/H}$ epochs and maintains version spaces $\TT_k$ and $\Theta_k$ over the model classes based on the data collected so far for each epoch $k \in [K']$.
In epoch $k$, \fnalg identifies a policy $\hat \pi_k$ and models $\hat T_k$ and $\hat O_k$ by solving an optimistic optimization problem over the version spaces. Here $v_{\MM(T, O)}(\pi)$ denotes the value of a policy $\pi$ in the POMDP model given by transition function $T$, emission function $O$ and the true reward function $r$, akin to the original definition in \eqref{eq::policy-value}. Still within epoch $k$, for each $h \in [H]$, \fnalg generates an exploration policy from $\hat \pi_k$ by taking $\tilde \pi_k = \hat \pi_k \circ_h \unif(\AA)$. The operator $\circ_h$ replaces $\hat \pi(\cdot | \tau_h)$ with the uniform distribution $\unif(\AA) \in \Delta(\AA)$ over the actions for all $h$-length histories $\tau_h$, but leaves the rest of $\hat \pi_k$ unaffected. That is, $\tilde \pi_k$ plays $\hat \pi_k$ normally up to the $h$th step and then takes a random action.

\fnalg deploys the exploration policy $\tilde \pi_k$ in the environment and records the observation $y_h^k$ and action $a_h^k$. Then, when the latent states are revealed, it records the latent state $x_h^k$ and the next state ${\tilde x_{h}^k}$. It repeats this exploration procedure for each $h \in [H]$ in the epoch $k$ to generate $(y_{1:H}^k, a_{1:H}^k, x_{1:H}^k, \tilde x_{1:H}^k)$. Based on this new data, it updates the version spaces via maximum likelihood estimation (MLE) with the following log-likelihood objectives:
\begin{align*}
    \hat \LL_k^1 (T) & = \sum_{\ell \in [k], h\in[H]} \log T(\tilde x^\ell_h | x^\ell_h, a^\ell_h) \quad \text{and} \\
    \hat \LL_k^2 (O) & =  \sum_{\ell \in [k], h \in [H]} \log O( y_h^\ell | x_h^\ell) .
\end{align*}

\subsection{Sample complexity bound}

We now state the performance guarantee for \fnalg. 
\begin{theorem}\label{thm::fn-approx}
Let $\MM$ be a \setshort model with a low-rank transition function $T_\star$ of rank $d$. Let $\TT$ and $\Theta$ satisfy Assumption~\ref{asmp::realizability}. Then, with probability at least $1 -  \delta$, \fnalg outputs a sequence of policies $\hat \pi_1, \ldots, \hat \pi_{K'}$ such that
\begin{align*}
 \regret(K')  
 = \OO \left(   \sqrt{ {   AH^4d K'}  \log \left(K'H | \Theta | |\TT |\over \delta \right) \log K'   }  \right)
\end{align*}
\end{theorem}
Note that the regret is over the learned $\hat \pi_{1:K'}$, not the actually deployed exploration policies. This phenomenon occurs often from one-step exploration~\citep{jiang2017contextual,agarwal2020flambe}. However, our focus is the implied PAC guarantee. Again, using a standard online to batch conversion, we get that \fnalg learns an $\epsilon$-optimal policy with probability at least $1 - \delta$ in 
\begin{align*}
    K = \widetilde \OO \left( {A H^5 d \over \epsilon^2  }  \log \left( {  | \Theta | |\TT |\over \delta }  \right)   \right) 
\end{align*}
episodes of interaction. The additional $H$ arises because there are $H$ episodes for each epoch $k \in [K']$ due to the construction and deployment of the exploration policies.

\begin{itemize}[nosep, leftmargin=10pt]
\item In contrast to the tabular bound of \tabalg, \fnalg has no dependence on the size of the latent state space $X$ or observation space $Y$. Instead, generalization using function classes replaces them with complexities of $\TT$ and $\Theta$ and the rank $d$ of the latent transition. Note that we can readily replace these log-cardinalities with other suitable notions of complexity for infinite function classes.
\item For comparison to Theorem~\ref{thm::tabular}, we can set $d=X$, $\log|\Theta| = \otil(XY)$ and $\log|\TT| = \otil(X^2A)$, which results in a suboptimal scaling in $X$ as \fnalg does not use value function-based optimism unlike \tabalg.
\item Similar to the tabular setting, the sample complexity also has no dependence on observability parameters, showing that we maintain this advantageous property of \setshort models in the function approximation setting.
    \item Observe that we have not made any further structural conditions on $O_\star$ to achieve this result. The structural condition is only on $T_\star$. 
\end{itemize}

\paragraph{Proof intuition.}
The proof is remarkably simple in contrast to the tabular result.
It is a combination of just two components. The first is a simulation lemma (Lemma~\ref{lem::sim-lem}), which relates the estimation error of the estimated value function to the total variation error of both $\hat T_k$ and $\hat O_k$. This decomposition allows us the analyze the error almost entirely in terms of the latent state distributions of the exploration policies, rather than their histories.  
This leads to the second component, which is a standard ``one-step-back'' analysis that has previously appeared for low-rank MDPs in the fully observable setting~\citep{agarwal2020flambe}. The use of uniform exploration policies at each $h$ is also common in MDP literature to handle the distribution shift between current and historical data. Here the randomness also plays a secondary role of removing all history dependence.

\section{Discussion}

Motivated by practical applications in partially observable problems, we formulated the problem setting of a \setting (\setshort), where the objective is to learn a decision making policy based on partial observations in order to interact with the environment, but the underlying latent states are eventually revealed during the training process. We proposed an algorithm, \tabalg, for finite latent states and observations. We gave sample-complexity upper and lower bounds, showing that \tabalg has no dependence on partial observability parameters and that it is nearly optimal in dependence on $XY$. We also proposed an algorithm, \fnalg, that allows for function approximation of the transition and emission functions to handle generalization in large or infinite latent state and observation spaces.

There are a number of interesting open directions for future work on hindsight observability. The similarities and compatibility between the \setshort and  standard MDP models make \setshorts a ripe area for further advancements that leverage our deeper knowledge of MDPs. Natural directions include, for example, model-free RL \citep{jin2018q,jin2020provably}, offline RL \citep{xie2021bellman,jin2021pessimism,zanette2021provable}, model selection \citep{lee2021online,lee2022oracle,cutkosky2021dynamic}, and computationally efficient representation learning for latent states \citep{agarwal2020flambe}.

Specific to function approximation, our most general results applied to low-rank latent transition functions for generalization and representation learning; however, MDP literature has had success with more general conditions that restrict the form of the latent state Bellman error \citep{jiang2017contextual,sun2019model, du2021bilinear,agarwal2022model}.
Unfortunately, these conditions have little meaning in \setshorts and POMDPs because partially observable value functions are history-dependent, not latent state-dependent. It would be interesting in the future to either reconcile these types of conditions or develop new meaningful ones for \setshorts.

\section*{Acknowledgements}

Part of this work was done while JNL was a student researcher at Google Research. JNL acknowledges support from the NSF GRFP.

\bibliography{refs}
\bibliographystyle{icml2023}

\newpage
\appendix

\onecolumn

\renewcommand{\contentsname}{Contents of Appendix}
\tableofcontents
\addtocontents{toc}{\protect\setcounter{tocdepth}{3}} 
\clearpage

\paragraph{Image sources:} The globe in Figure~\ref{fig::problem} is taken from \url{https://commons.wikimedia.org/wiki/File:Ambox_globe_Americas.svg}.

\section{Additional Motivating Applications}
\label{app:more_applications}

To further establish hindsight observability, we describe in more detail a number of motivating applications. The diversity of the problem settings highlights the generality of the \setshort model.
\begin{itemize}
    \item \textbf{Sim-to-real robotics.} Sim-to-real (simulation to reality) is a well-studied paradigm in robotics in which one trains a robot using RL in a simulator with the end goal of deploying the robot in the real world at test time. The test time policy often uses high dimensional, partial observations such as images from a camera (which are susceptible to noise and occlusions). However, during training, the simulator grants access to the underlying state (object positions, poses, etc.), information that is not available at test time but might dramatically increase sample efficiency. Empirically, leveraging this train-time information has led to improved sample complexity \citep{pinto2017asymmetric,chen2020learning}.
    
    \item \textbf{High-frequency control.} In a related control setting, latency and computational bottlenecks (e.g. processing lidar data, depth images, etc.) can obscure and delay observation of the true state  even though control inputs must be made quickly. It is common in such settings to instead use simpler observations, such as images from a standard camera. Once the system finally observes or processes the true states, the states can be incorporated in the training procedure. \citet{pan2017agile} explored this direction, successfully training an autonomous rally car via distillation for high-speed driving.
    
    \item \textbf{Meta-learning and latent MDPs.} In meta-learning for RL \citep{wang2016learning,finn2017model}, an agent leverages past rollouts on different MDP tasks sampled from a fixed distribution, where the reward and transition functions differ from task to task. The training tasks are often labeled or can be inferred in hindsight \citep{liu2021decoupling}. At test time, the task is unknown and thus the agent  must adapt using only observations to infer and maximize reward for the test time task. This can be modeled as a \setshort, where the latent states are the MDP states concatenated with the task context and the observations are the MDP states. Meta-learning with finitely many tasks can also be viewed as a latent MDP~\citep{kwon2021rl,zhou2022horizon}.
    
    \item \textbf{Scheduling.} Following \citet{sinclair2022hindsight}, consider a data center, which aims to allocate submitted jobs to servers efficiently. The arrival times of the jobs and their total lengths are unknown. However, once a job has completed, both its arrival time and length are known in hindsight. The latent states of the equivalent \setshort are the observations concatenated with the arrival times and lengths of all jobs.

    \item \textbf{Online imitation learning} In online imitation learning \citep{ross2011reduction}, an agent interacts with the MDP over $K$ rounds, executing actions under its learned policy. After a round, it retroactively queries an expert for optimal actions in the visited states to then update the policy. This can be viewed as a special case of our setting where each of our latent states is a visited state concatenated with the expert's optimal action for that state, which is only retroactively observed. 
    
    \item \textbf{Screening for diseases.} Medical screenings are procedures designed to help detect and monitor diseases in patients, usually in early stages. Screenings are typically dependent on patient characteristics as well as tests, which may not always be accurate and may have undesirable effects. \citet{ayer2012or} frame the problem of screening for breast cancer as a POMDP, where the state is a patient's true condition (progression of the disease), and the observations are the outcomes of tests such as a mammogram. Depending on the actions taken in response to the observed tests, the patient may eventually undergo a perfect test revealing the latent state (such as a biopsy). Revelation of the state can then be used to more effectively guide observation-based policies in the future by allowing association between the outcomes of prior tests and the true condition.
\end{itemize}

\section{Value Functions and Alpha Vector Representations}\label{app::alpha-vectors}

The proofs in this paper crucially rely on the $\alpha$-vector representation of value functions for POMDPs \citep{smallwood1973optimal}. Since these techniques and intuitions are not common in reinforcement learning theory literature\footnote{Exceptions include \citet{kwon2021rl,zhou2022horizon} who used $\alpha$-vectors in the latent MDP model (single unobserved context).}, we now give an introductory treatment of this topic with all the tools necessary for our proofs.

Since we work with history-dependent policies, the value $V^\pi$ and action-value functions $Q^\pi$  of a policy $\pi$ are history-dependent as well. The history of observations and actions $\tau_h$ gives rise to a posterior distribution $b_h(\cdot) = P(x_h = \cdot | \tau_h)$ over the latent states.\footnote{The posterior depends on the initial latent state distribution $\rho$ but we omit a notational references to it for clarity.} We could define value functions simply as a function of $\tau_h$ but it will be convenient to make the posterior $b_n$ explicit and write value functions as a function of both $x_h$ and $b_h$:
\begin{align*}
	Q^\pi_h(b_h, \tau_h, a_h)  
& = \E_{x_h \sim b_h}\!\! \left[ \!\E\!\left[ \sum_{h' = h}^H r(x_h, a_h)  \  | \ \tau_h, x_h, a_h \right] \! \right] \\
	V^\pi_h(b_h, \tau_h)  & = \sum_{a_h} \pi(a_h | \tau_h) Q^\pi_h( b_h, \tau_h , a_h).
\end{align*}
While $b_h$ determines the latent state, $\tau_h$ affects the actions taken by the policy.

The $\alpha$-vectors of a POMDP act as a kind of representation for the value functions of possible policies on the POMDP. This allows us to represent value functions as linear functions of belief vectors.
Note that it is not obvious how to write either of the value function $V$ or $Q$ for POMDPs in a concise, recursive form in the same way that we do for MDP value functions. The $\alpha$-vector representation provides an alternative solution where we can represent the value functions as linear functions with the $\alpha$-vectors and then have a recursive definition of the $\alpha$-vectors.

Consider a POMDP with transtion function $T$ and emission function $O$.
Consider the last timestep $H$ and a belief distribution $b \in \Delta(\XX)$ over the latent state:
\begin{align}
	Q_H^\pi(b, \tau, a) & = \sum_x b(x) r(x, a) \\
	V_H^\pi(b, \tau) & = \sum_{x, a} b(x) \pi(a | \tau)  r(x, a). 
\end{align}
These clearly have simple linear representations as functions of the belief vector:
\begin{align}
	Q_H^\pi(b, \tau, a)  & = b^\top \alpha^\pi_{H, \tau}(\cdot, a)  \\
	V_H^\pi(b, \tau) & = b^\top \alpha^\pi_{H, \tau}(\cdot),
\end{align}
where $\tau$ is a partial history of appropriate length, $\alpha^\pi_{H, \tau} (\cdot) \in \R^X$ and  $\alpha^\pi_{H, \tau}(\cdot, \cdot) \in \R^{X \times A}$ are given by 
\begin{align}
	\alpha^\pi_{H, \tau} (x, a)  & = r(x, a) \\
	\alpha^\pi_{H, \tau} (x, a) & = \sum_{a} \pi(a | \tau) \alpha^\pi_{H, \tau} (x, a).  
\end{align}
Let us now assume inductively that $V^\pi_{h + 1}(x)$ and $Q^\pi_{h + 1}$ have the following representations:
\begin{align}
	Q^\pi_{ h+ 1} (b, \tau, a) & = b^\top \alpha^\pi_{h + 1, \tau} (\cdot, a) \\
	V^\pi_{ h+ 1} (b, \tau) & = b^\top \alpha^\pi_{h + 1, \tau}.
\end{align}
From the definition of $Q$,
\begin{align}
	Q_{h}^\pi(b,\tau, a) = b^\top r(\cdot,a ) + \sum_{y'} P(y' | \tau, a) V^\pi_{h + 1} (b', \tau')
\end{align}
 where we  denote $\tau'$ as the history $\tau$ concatenated with the new action $a$ and observation $y'$ and where $b'$ (which is dependent on $y'$ and $a$ ) is updated belief vector starting from  $b$ given action $a$ and the next observation $y'$.
\begin{align}
	b'(x') := { O(y' | x') \sum_{x} b(x) T(x' | x, a) \over  \sum_{x''  } O(y' | x'' ) \sum_{x} b(x) T(x'' | x, a) } .
\end{align}
The action-value function can then be rewritten as 
\begin{align}
	Q_{h}^\pi(b,\tau, a) & = b^\top r(\cdot ,a ) + \sum_{y'} P(y' | \tau, a) \sum_{x'} b'(x') \alpha^\pi_{h+1, \tau'}(x') \\
	& = b^\top r(\cdot ,a ) + \sum_{x, x', y'} b(x) T(x' | x, a) O(y' | x' ) \alpha^\pi_{h+1, \tau'}(x') \\
	& = b^\top r(\cdot ,a ) + b^\top \gamma^\pi_{h, \tau}(\cdot, a),
\end{align}
where we define
\begin{align}
	\gamma^\pi_{h, \tau}(x, a) & = \sum_{x',  y'} T(x' | x, a) O(y' |  x') \alpha^\pi_{h + 1, \tau'}(x') \\
	\gamma^\pi_{h, \tau}(x) & = \sum_{a, x', y'} \pi_h(a | \tau)  T(x' | x, a) O(y'|  x') \alpha^\pi_{h + 1, \tau'}( x').
\end{align}
Then, define
\begin{align}
	\alpha^\pi_{h, \tau}(\cdot, a) & = r(\cdot ,a ) + \gamma^\pi_{h, \tau}(\cdot, a) \\
	\alpha^\pi_{h,\tau}(\cdot) & = \sum_{a} \pi_h(a | b) r(\cdot ,a) + \gamma^\pi_{h, \tau}(\cdot).  
\end{align}
This enables
\begin{align}
	Q^\pi_h(b, \tau, a) & = b^\top \alpha^\pi_{h, \tau}(\cdot, a) \\
	V^\pi_h(b, \tau) & = b^\top \alpha^\pi_{h, \tau}(\cdot).
\end{align}
We may repeat this recursion until the end $h = 1$. 

This culminates in the following $\alpha$-vector proposition.

\begin{restatable}[$\alpha$-vector representation]{proposition}{alphavector}
	\label{prop::alpha}
	Let $\pi$ be a fixed history-dependent policy. Let $b_h$ denote the belief vector (posterior distribution over $\XX$) given the history $\tau_h = (y_{1:h}, a_{1:h - 1})$. Then, there exist vectors $\alpha^\pi_{h, \tau_h}(\cdot) \in \R^X$ and $\alpha^\pi_{h, \tau_h}(\cdot ,a) \in \R^X$ such that, for all $(h, x, a,\tau_h)$, the following equations hold:
	\begin{align}
		V^\pi_h(b_h, \tau_h) & = b_h^\top \alpha^\pi_{h, \tau_h}(\cdot) \\
		Q^\pi_h(b_h, \tau_h) & = b_h^\top \alpha^\pi_{h, \tau_h}(\cdot, a) \\
		\alpha^\pi_{h, \tau_h}(x, a) & = r(x, a) + \sum_{x',  y'} T(x' | x, a) O(y' |  x') \alpha^\pi_{h + 1, \tau'}(x') \\
		\alpha^\pi_{h, \tau_h}(x) & = \sum_{a} \pi(a | \tau_h) \alpha^\pi_{h, \tau_h}(x, a),	 
	\end{align}
	where $\tau'$ is the concatenation of $\tau$ with observation $y'$ and action $a$, and $\alpha_{H + 1} = 0$.
	Furthermore, $\max \left\{ \alpha^\pi_{h, \tau_h} (x, a), \alpha^\pi_{h, \tau_h}( x) \right\} \leq G(H - h + 1)$ if $r(x, a) \in [0, G]$ for all $x \in \XX$ and $a \in \AA$ and $h \in [H]$.
\end{restatable}

\begin{proof}
	We have already proved the equations by induction. It remains to show that the values we constructed satisfy the last statement, the bound. We will focus on $\alpha^\pi_{h, \tau}(x, a)$ since it is clear that if the bound is satisfied for this one, then it  is satisfied for $\alpha^\pi_{h, \tau}( x )$. Using proof by induction, we have the base case $\alpha^\pi_{H, \tau}(x,a) \leq \max_{x, a} r(x, a) \in [0, G]$. Then,
	\begin{align}
		\alpha^\pi_{h, \tau}(x, a) & = r(x, a) + \sum_{x', y'}  T(x' | x, a) O(y' | x') \alpha^\pi_{h + 1, \tau'} (x')  \\
		& \leq r(x, a) + \sum_{x', y'}  T(x' | x, a) O(y' | x')  G \left(H - (h + 1) + 1 \right) \\
		& = r(x, a) + G(H - h) \\
		& \leq G(H - h +1).
	\end{align} 
	This concludes the proof.

\end{proof}

\section{Full Statement and Proof of Theorem~\ref{thm::tabular}} \label{app::tabular}

Having established the notation and important concepts of the $\alpha$-vectors, we now proceed to the full statement (including all lower order terms) and the proof of Theorem~\ref{thm::tabular}, which will immediately put these concepts to use.

\subsection{Full statement of result of Theorem~\ref{thm::tabular}}
We let $a \lesssim b$ mean that $a \leq  c b$ where $c$ is a problem-independent constant.

\begin{theorem}
Let $\MM$ be a \setshort model with $X$ latent states and $Y$ observations.
With probability at least $1 - 5\delta$, \tabalg outputs a sequence of policies $\hat \pi_1, \ldots, \hat \pi_K$ such that
\begin{align}
	\sum_{k \in [K]} v(\pistar) - v(\hat \pi_k)  & \lesssim   \underbrace{    \sqrt{H^5 K \log (2/\delta)} }_{\text{Azuma-Hoeffding}} %
	+ \underbrace{  \sqrt{  Y X H^5 K \iota } }_{\text{Emission error}} 
	+  \underbrace { \sqrt{ X A H^4 K \iota }   +     H^4 X^2 A \iota (1 + \log(K)) }_{ \text{Transition error} } \\
	&\quad +  \underbrace{ H^3 X\sqrt{  Y \iota  }  +  H X A\sqrt{ H^3 \iota }   }_{\text{Residual pigeonhole error}},
	\end{align}
	where $\iota = \log (\nicefrac{2X^2 Y A KH }{ \delta})$. 
\end{theorem}

\subsection{High-probability events}
We begin by defining several events that we show occur with high probability. 
For the optimal policy $\pistar$,  there exist vectors $\alpha^{\pistar}_{h, \tau_h} \in \R^X$, indexed by latent states, such that $V^{\pistar}_h(b_h, \tau_h) = b_h^\top \alpha^{\pistar}_{h, \tau_h}$ (see  Proposition~\ref{prop::alpha}). As such, we define the following event that bounds the deviation on $T_\star$ by leveraging $\alpha^{\pistar}$ in a similar manner to how the optimal value functions are leveraged in improved MDP analyses such as \citet{azar2017minimax}. Define
\begin{align}
\EE_{T} = \defineHighProbT,
\end{align}
where $C_T = \constantCT$.
For $h = H$, the left side is just zero. Recall that $\tau_h'$ denotes the concatenated partial history of $\tau_h$ with $(y', a)$ (i.e. the ``next-step" partial history.
The purpose of this event is to avoid resorting to a total variation bound which necessitates immediate dependence on $X$. By instead bounding quantities solely in terms of the $\alpha$-vector of the optimal policy, $\alpha^{\pistar}_{h + 1, \tau_h'}(x')$, we can still produce valid bonuses. Such tricks are often used to get the best MDP regret bounds \cite{azar2017minimax}.
However, in contrast to the MDP style analyses, the above event must hold across all possible observable histories $\tau_h$ which leads to additional polynomial dependence on $H$ (due to there being exponentially many histories in $H$ and a union bound over these histories) and polylogarithmic dependence on $Y$. However, using this approach will save a $X$ factor in the sample complexity bound.

We also consider the following alternative version of $\EE_T$. Let $c \geq 1$ be a constant, potentially dependent on $H$. Define
\begin{align}\label{eq::transition-event-c}
 \EE_T^c & = \defineHighProbC.
\end{align}

To handle estimation of the emission matrix, we will use a more conventional event based on the total variation difference of $O_\star$ and the estimated quantity $\hat O_k$:
\begin{align}
	\EE_O = \defineHighProbTab%
	,
\end{align}
where $C_O = \constantCO$.

\begin{lemma}
	$P(\EE_T) \geq 1- \delta$.
\end{lemma}
\begin{proof} For each $x, a$ we will assume that $KH$ independent transitions are preemptively sampled from $T(\cdot | x, a)$ and revealed in order to the learner with each visit to $(x, a)$, since this is distributionally identical to the interface the learner encounters. Let $\hat T_{(n)}(\cdot | x, a)$ denote the empirical distribution estimated with the first $n \in [KH]$ samples.
We can then apply Hoeffding's inequality to a specific sum of independent variables. Consider fixed $n \in [KH], x, a$ and an arbitrary function $g: \XX \to \R$ such that $ \| g\|_\infty \leq G$ uniformly for some $G \geq 0$. Here, we will use $x'_{(i)}$ to denote the realizations of the preemptive samples for $i \in [n]$. Then, with probability at least $1 - \delta$,
	\begin{align}
	\sum_{ x'}  g(x') \left( \hat T_{(n)} (x'  | x, a) -  T(x' | x, a) \right) & = {1 \over n} \sum_{i \in [n]}   \left(  \sum_{x'} g(x') \left(  \1 \{x'_{(i)} = x'\} -  T_\star(x' | x, a) \right) \right). 
	\end{align}
	Since this is a sum of independent random variables bounded within $[-G, G]$, Hoeffding's inequality implies that 
	\begin{align}
	\sum_{ x'}  g(x') \left( \hat T_{(n)} (x'  | x, a) -  T_\star(x' | x, a) \right) & \leq  2G \sqrt{ \log(1/\delta)  \over 2 n}
	\end{align}
	with probability at least $1 - \delta$.  Then, we can simply choose $g(x') = \sum_{y'} O_\star(y' | x') \alpha^{\pi_*}_{h + 1, \tau_h'}(x')$, which is fixed and has $|g(x')| \leq  (H - h)$ via the bound on the size of the $\alpha$-vectors due to Proposition~\ref{prop::alpha}. Taking the union bound over all $n \in [KH], h \in [H], x \in \XX, a \in \AA, \tau_h \in \YY^h \times \AA^{h - 1}$, we have
	\begin{align*}
	    2(H - h) \sqrt{ \log\left(Y^h X A^{h - 1} H^2 K  /\delta\right)  \over 2 n} \leq   \sqrt{ C_T H^3 \log \left(YXAHK/\delta\right)},
	\end{align*}
	which gives the result with constant $C_T = \constantCT$.
\end{proof}

\begin{lemma}
	$P(\EE_T^c) \geq 1 - \delta$.
\end{lemma}
\begin{proof}
	We again use the distributionally equivalent notation from the prior proof and the same notation for $n$.
	By Bernstein's inequality (Lemma~\ref{lem::bernstein}), with probability at least $1 - \delta$,
	\begin{align}
	\hat T_{(n)}(x' | x, a) - T_\star(x' | x, a)
	& \leq {T_\star(x' | x, a) \over 2c } + {2c  \log(1/\delta) \over n }.
	\end{align}
	Taking the union bound over all $n \in [KH]$, $x, x' \in \XX$, and $a \in \AA$ gives the result.
\end{proof}

$ \EE^c_T$ immediately implies the following error bound.

\begin{corollary}\label{cor::transition-mean}
	Let $g : \XX  \to [-G, G]$ be a bounded function for $G \geq 0$. Suppose that $ \EE_T^c$ holds. Then, for all $x, a, k$, 
	\begin{align}
	\sum_{x'} \left(\hat T_k(x' | x,a ) - T_\star(x' | x, a) \right)  g(x')  & \leq {1 \over 2c }\sum_{x'} T_\star(x' | x, a) g(x') +  {2c G X \log(X^2 AKH /\delta) \over n_k(x, a)} .
	\end{align}
\end{corollary}
\begin{proof}
	The proof is immediate by rearranging the definition of $ \EE^c_T$.
\end{proof}

\begin{lemma}\label{lem::emission-tab}
	$P(\EE_O) \geq 1-  \delta$.
\end{lemma}
\begin{proof}
	A similar approach via Bernstein's inequality (Lemma~\ref{lem::bernstein}) guarantees the following for the analogously defined $\hat O_{(n)} (y | x)$ for $n \in [KH]$. With probability at least $1 - \delta$, for all $x, y, n$, 
	\begin{align}
		| \hat O_{(n)} (y | x) - O_\star(y | x) | \leq \sqrt{  2 O_\star(y | x) \iota  \over n } + {   \iota  \over 3 n },
	\end{align}
	where $\iota = \log(2 YX KH/\delta)$. 
	Therefore, 
	
	\begin{align}
		{1 \over 2} \|  \hat O_{(n)} (\cdot  | x) - O_\star(\cdot | x) \|_1 & \leq {   Y  \iota  \over 6 n } + \sum_y  \sqrt{   O_\star(y | x) \iota  \over  2n }   \\
		& \leq {   Y  \iota  \over 2 n } + \sqrt{   Y \iota  \over  2n } \\
		& \leq \sqrt{ 2 Y \iota \over n},
	\end{align}
	where the last line follows from the fact that ${1 \over 2} \| \hat O_{(n)} (\cdot  |x) - O_\star(\cdot | x) \|_1 \leq 1 $ always. So if the upper bound of the right side is at most $1$ then ${ Y \iota \over 2n }$ is at most $1$. Therefore, the desired bound holds with $C_O = \constantCO$.
\end{proof}

\subsection{Optimism via reward bonuses}

We let $\hat \alpha_k$ and $\alpha$ denote the $\alpha$-vectors under the learned model $\widehat\MM_k$, which uses the transition function $\hat T_k$ and emission function $\hat O_k$ and bonus reward function $\hat r_k$, and the true model $\MM$, respectively.

\begin{lemma}\label{lem::alpha-optimism-refined} Suppose that $\EE_T$ and $\EE_O$ hold.
	For all $k, h, \tau_h, x$, it holds that $\hat \alpha^{\pistar}_{k, h, \tau_h}(x) \geq  \alpha^{\pistar}_{h, \tau_h}(x) +  H \epsilon_k(x)$ for all $h \in [H]$.
\end{lemma}
\begin{proof}
	The proof is by induction on $h$. Let $k$ be fixed so we can drop the subscript notation for it. Observe that we clearly have the base case via Proposition~\ref{prop::alpha}:
	\begin{align}
	\hat \alpha^{\pistar}_{H, \tau_H} (x, a) & = \hat r(x,a) \\
	& =  r(x,a ) + H \epsilon(x) + \epsilon(x, a) \\
	& = \alpha^{\pistar}_{H, \tau_H} (x, a) + H \epsilon(x) + \epsilon(x, a) \\
	& \geq \alpha^{\pistar}_{H, \tau_H} (x, a) + H \epsilon(x).
	\end{align} 
	Fix $h \in [H - 1]$. Recall the definition of $\tau_h'$ as the ``next-step" partial history. Assume that $\hat \alpha^{\pistar}_{ h + 1, \tau_h'}(x) \geq \hat \alpha^{\pistar}_{ h + 1, \tau_h'}(x) + H \epsilon(x)$. Then,
	\begin{align}
	\alpha^{\pistar}_{h, \tau_h}(x, a) & = r(x, a) + \sum_{x' , y'} T_\star(x' | x, a) O_\star(y' | x')  \alpha^{\pistar}_{h + 1, \tau_h' }(x') \\
	& =  r(x,a ) + \sum_{x', y'} O_\star(y' | x')   \left( T_\star (x' | x, a) - \hat T_\star(x' | x, a)  \right) \alpha^{\pistar}_{h + 1,\tau_h'}(x') \\
	& \quad + \sum_{x', y'} \hat T(x' | x, a)  \left(  O_\star(y ' | x' ) - \hat O(y' | x')  \right)\alpha^{\pistar}_{h + 1,\tau_h'}(x')   \\
	& \quad + \sum_{x', y'} \hat T(x' | x,a) \hat O(y' | x') \alpha^{\pistar}_{h + 1,\tau_h'}(x').
	\end{align}
	The first summation is bounded using $\EE_T$:
	\begin{align}
	\sum_{x', y'} O_\star(y' | x')   \left( T_\star (x' | x, a) - \hat T(x' | x, a)  \right) \alpha^{\pistar}_{h + 1,\tau_h'}(x') & \leq \sqrt{ C_T  H^3 \log (YX A HK  /\delta ) \over n(x, a) }  \\
	& \leq  \epsilon(x, a). 
	\end{align}
	The second summation can be bounded in terms of the total variation distance for the emission matrices along with $\EE_O$:
	\begin{align}
	\sum_{x', y'} \hat T(x' | x, a)  \left(  O_\star(y ' | x' ) - \hat O(y' | x')  \right)\alpha^{\pistar}_{h + 1,\tau_h'}(x') & \leq \sum_{x'} (H - h)  \hat T(x' | x, a) \| O_\star(\cdot | x') - \hat O(\cdot | x') \|_1 \\
	& \leq \sum_{x'} (H - h)  \hat T(x' | x, a) \epsilon(x'),
	\end{align}
	which also uses Proposition~\ref{prop::alpha} to bound the magnitude of $\alpha^{\pistar}$. Finally, for the third summation, we can use the inductive hypothesis to get
	\begin{align}
	\sum_{x', y'} \hat T(x' | x,a) \hat O(y' | x') \alpha^{\pistar}_{h + 1,\tau_h'}(x') & \leq \sum_{x', y'} \hat T(x' | x,a) \hat O(y' | x') \left( \hat \alpha^{\pistar}_{h + 1,\tau_h'}(x') - H  \epsilon(x') \right).
	\end{align}
	Combining these three individual bounds, we have
	\begin{align}
\alpha^{\pistar}_{h, \tau_h}(x, a) & \leq  r(x, a) +  \epsilon(x, a)  + \sum_{x'} (H - h) \hat T(x' | x, a)  \epsilon(x') \\
& \quad + \sum_{x', y'} \hat T(x' | x,a) \hat O(y' | x') \left( \hat \alpha^{\pistar}_{h + 1,\tau_h'}(x') - H  \epsilon(x') \right) \\
& \leq r(x, a) +  \epsilon(x, a) + \sum_{x', y'} \hat T(x' | x,a) \hat O(y' | x') \hat \alpha^{\pistar}_{h + 1,\tau_h'}(x')  \\
& = \hat \alpha^{\pistar}_{h, \tau_h}(x, a) - H \epsilon(x),
	\end{align}
where we have added and subtracted $H \epsilon(x)$ and used the definition of $\hat \alpha^{\pistar}_{h, \tau_h}$ (Propostion~\ref{prop::alpha}) in the last step. Applying this inductively gives the result.

\end{proof}

\begin{lemma}\label{lem::learned-alpha-bound}
	For all history-dependent $\pi$ and all $k, h, \tau_h$, it holds that $\| \hat \alpha^\pi_{k, h, \tau_h} \|_{\infty} \leq  5 H (H - h + 1)$.
\end{lemma}
\begin{proof}
	As before, we assume $k$ is fixed and drop subscript notation for it. Note that we have $\hat r(x, a) = r(x, a) + H \epsilon(x) + \epsilon(x, a)$. Furthermore,
	\begin{align}
	\epsilon(x) & \leq 2 \\
	\epsilon(x, a) & \leq 2H.
	\end{align} 
	Therefore, $\hat r(x, a) \in [0, 5 H]$.  Proposition~\ref{prop::alpha}  implies the result.
\end{proof}

\subsection{Proof of the theorem}

Define the intersection of the above events as $\EE = \EE_T \cap \EE_O \cap  \EE^c_T \cap  \EE^1_T$ for some fixed $c$ to be determined later. Note that $P(\EE) \geq 1 - 4\delta$ by the union bound.
For the remainder of the proof, we shall assume that these four hold simultaneously. Recall also that we define the errors as
	\begin{align}
	\epsilon_k(x, a) & := \min \left\{  2 H, \  \sqrt{C_T H^3 \log (Y X A H K / \delta ) \over n_k(x, a) } \right\}, \\
	\epsilon_k(x) & := \min \left\{  2, \  \sqrt{ C_O Y \log ( Y X K H /\delta ) \over n_k(x)} \right\}.
	\end{align}
	We define an additional error term that is not used in the algorithm, only the analysis:
	\begin{align}
	\tilde \epsilon_k(c, x, a) := \epsilontildek.
	\end{align}
	
	We begin by analyzing a fixed round $k$ and will temporarily drop subscripts denoting $k$.
	We will use $\hat \E$ and $\hat P$ to denote expectation and probabilities under the learned model $\widehat \MM$ during this round, as defined in the previous subsection. We let $\hat v(\pi)$ denote the average value of the policy $\pi$ under $\widehat \MM$, akin to the true value $v(\pi)$. Then, using the $\alpha$-vector definition of the value functions,
	\begin{align}
	\hat v(\pistar) - v(\pistar) & = \sum_{x, y} \rho(x) \hat O(y | x)  \hat \alpha^{\pistar}_{1, \tau_1}(x) - \sum_{x, y} \rho(x) O_\star(y | x)  \alpha^{\pistar}_{ 1, \tau_1} (x) \\
	& \geq -  H \sum_{x} \rho(x) \epsilon(x) + \sum_{x, y} \rho(x) \hat O(y | x)  \left( \hat \alpha^{\pistar}_{1, \tau_1} (x) -  \alpha^{\pistar}_{1, \tau_1}  (x)  \right) \\
	& \geq 0,
	\end{align}
	where the last inequality follows from the optimism bound in Lemma~\ref{lem::alpha-optimism-refined}.
	Therefore, by $\EE_O$ and Lemma~\ref{lem::learned-alpha-bound},
	\begin{align}
	v(\pistar) - v(\hat \pi)  
	& \leq  \hat v(\hat \pi) - v(\hat \pi) \\
	&  = \sum_{x, y} \rho(x)  \hat O(y | x) \hat \alpha^{\hat \pi}_{1, \tau_1}(x) - \sum_{x, y} \rho(x)  O_\star(y | x)  \alpha^{\hat \pi}_{1, \tau_1}(x) \\
	& \leq  5 H^2 \E_{\hat \pi} \left[  \epsilon(x_1)  \right] + \E_{\hat \pi} \left[  \hat \alpha_{1, \tau_1}^{\hat \pi}(x_1) - \alpha^{\hat \pi}_{1, \tau_1} (x_1) \right]   \label{eq::eq1}
	\end{align} The crux of the proof lies in the following lemma which recursively bounds the expected differences of the $\alpha$-vectors under $\hat \pi$.
	For convenience, let us set $\precConst := \precConstValue$.

	\begin{lemma} \label{lem::recursive-alpha-bound}
		Let $h \in [H - 1]$ be fixed. Then, under the aforementioned events, it holds that
		\begin{align}
		\E_{\hat \pi}  \left[  \hat \alpha_{h, \tau_h}^{\hat \pi} (x_h) - \alpha_{h, \tau_h}^{\hat \pi} ( x_h)  \right]  & \leq \E_{\hat \pi} \left[   H \epsilon(x_h) + 2 \epsilon(x_h, a_h) + \precConst H^2 \tilde \epsilon(c, x_h, a_h)  \right] \\
		& \quad + \left(  1 + {1 \over 2c } \right) \E_{\hat \pi}  \left[  \hat \alpha_{h + 1, \tau_{h + 1}}^{\hat \pi}(x_{h + 1}) - \alpha_{h + 1, \tau_{h + 1}}^{\hat \pi}(x_{h + 1})  + 11 H^2 \epsilon(x_{h + 1})\right].
		\end{align}
	Furthermore,
	\begin{align}
	    \E_{\hat \pi}  \left[  \hat \alpha_{H, \tau_H}^{\hat \pi} (x_H) - \alpha_{H, \tau_H}^{\hat \pi} ( x_H)   \right] & \leq \E_{\hat \pi} \left[ H \epsilon(x_H) +  \epsilon(x_H, a_H)\right]. 
	\end{align}
	\end{lemma}
	
	Lemma~\ref{lem::recursive-alpha-bound} is proved in Appendix~\ref{app::tabular-supporting}.

	 Now we can apply the bound from Lemma~\ref{lem::recursive-alpha-bound} recursively to get the following bound on the value difference under the true and learned models:
	\begin{align}
	\hat v(\hat \pi) - v(\hat \pi) \leq      \left( 1 + {1 \over 2c}  \right)^H\E_{\hat \pi} \left[  \sum_{h } 12 H^2 \epsilon(x_h) + 2 \epsilon(x_h, a_h) + \precConst H^2 \tilde \epsilon(c, x_h, a_h)  \right].
	\end{align}
	Therefore, we can choose $c =  \nicefrac{H}{2}$ to get
	\begin{align}
	\hat v(\hat \pi) - v(\hat \pi) & \leq   \precConst e \cdot  \E_{\hat \pi} \left[  \sum_h   H^2 \epsilon(x_h ) +  \epsilon(x_h, a_h) +  H^2 \tilde \epsilon(\nicefrac{H}{2}, x_h, a_h)  \right].
	\end{align}
	
	Summing over all $k \in [K]$,
	\begin{align}
	  \sum_{k \in [K]} v(\pistar) - v(\hat \pi_k)   & \leq \precConst e   \cdot  \sum_{k \in [K]} \E_{\hat \pi}  \left[\sum_{ h \in [H]} H^2 \epsilon_k(x_h) + \epsilon_k(x_h, a_h) +  H^2 \tilde \epsilon_k(\nicefrac{H}{2}, x_h, a_h) \right]. 
	\end{align}
	
	To bound this quantity with the pigeonhole principle, we apply the Azuma-Hoeffding bound (Lemma~\ref{lem::azuma}) to the martingale difference sequence defined with $Z_{k, h} :=   H^2\epsilon_k(x_h) + \epsilon_k(x_h, a_h) + H^2 \tilde \epsilon_k(\nicefrac{H}{2}, x_h, a_h)$ where $| Z_{k, h} - \E_{\hat \pi_k} \left[  Z_{k, h} \right] | \leq 12 H^2$. Therefore, under this additional event,
	\begin{align}
	& \sum_{k \in [K]} v(\pistar) - v(\hat \pi_k) \\  & \leq 48  \cdot \precConst e  \cdot H^2 \sqrt{ K H \log (2/\delta)} + \precConst e \cdot \sum_{k, h} H^2\epsilon_k(x_h) + \epsilon_k(x_h, a_h) + H^2\tilde \epsilon_k(\nicefrac{H}{2}, x_h, a_h) \\
	& \leq  48  \cdot \precConst e  \cdot H^2 \sqrt{K H \log (2/\delta)} + \precConst e \cdot \sum_{k, h} H^2 \sqrt{ C_O Y \iota  \over n_k(x^k_h) \vee 1 }  +  \sqrt{C_T H^3 \iota  \over n_k(x^k_h, a^k_h) \vee 1}   +  H^2  \cdot {  2 (H/2)   X \iota \over n_k(x^k_h, a^k_h) \vee 1 } 
	\end{align}	
	where $\iota = \log (\nicefrac{2X^2 Y A KH }{ \delta})$. Applying the pigeonhole principle the summations (Lemmas~\ref{lem::pigeon} and~\ref{lem::pigeon2}), we have
	\begin{align}
	\sum_{k \in [K]} v(\pistar) - v(\hat \pi_k)  & \leq    48  \cdot \precConst e  \sqrt{ H^5 K \log (2/\delta)} + 3\cdot \precConst e  \sqrt{C_O  Y X H^5 K \iota } + 3 \cdot \precConst e \sqrt{C_T X A H^4 K \iota }  \\
	& \quad + \precConst e  H^4 X^2 A \iota (1 + \log(K))   + \precConst e \left( H^2\sqrt{ C_O Y \iota  } H X  + \sqrt{C_{T} H^3 \iota } H X A  \right) 
	\end{align}
	We recall that the constants have values $C_T = \constantCT$, $C_O = \constantCO$, and $\precConst = \precConstValue$.
	Finally we conclude that the intersection of the good events $\EE$ and the Azuma-Hoeffding event occur simultaneously with probability at least $1 - 5\delta$.

\subsection{Supporting results} \label{app::tabular-supporting}

\subsubsection{Proof of Lemma~\ref{lem::recursive-alpha-bound}}  
\begin{proof}[Proof of Lemma~\ref{lem::recursive-alpha-bound}]

The second claim follows simply by the definition of $\hat r$ and using the form of $\hat \alpha_{H}^{\hat \pi}$ and $\alpha_H^{\hat \pi}$ at step $H$. We focus on the first claim.
	Note that, from the recursive definitions of $\hat \alpha$ and $\alpha$ in Proposition~\ref{prop::alpha}, we have
	\begin{align}
	\hat \alpha_{h, \tau_h}^{\hat \pi} (x, a) - \alpha_{h, \tau_h}^{\hat \pi} (x, a)  & = \hat r(x, a) - r(x, a) + \hat \BB_{\tau_h} (x, a) \left[  \hat \alpha^{\hat \pi}_{h + 1} \right]  - \BB_{\tau_h} (x, a) \left[  \alpha^{\hat \pi}_{h +1 }  \right],
	\end{align}
	where we define the operators
	\begin{align}
	\BB_{\tau} (x, a)  \left[ \alpha  \right] & : = \sum_{x', y'} O_\star(y' | x') T_\star(x' | x, a) \alpha_{\tau'}(x') \\
	\hat \BB_{\tau} (x, a)  \left[ \alpha  \right] & : = \sum_{x', y'} \hat O(y' | x') \hat T(x' | x, a) \alpha_{\tau'}(x') 
	\end{align}
	and we use the same convention of defining $\tau'$ as the concatenation of $\tau$ and $(a, y')$. Then,
	\begin{align}
	\hat \alpha_{h, \tau_h}^{\hat \pi} (x, a) - \alpha_{h, \tau_h}^{\hat \pi} (x, a) & = \hat r(x, a) - r(x, a) + \hat \BB_{\tau_h} (x, a) \left[  \hat \alpha^{\hat \pi}_{h + 1} \right]  - \BB_{\tau_h} (x, a) \left[  \alpha^{\hat \pi}_{h +1 }  \right] \\
	& = H \epsilon(x) + \epsilon(x, a)  + \hat \BB_{\tau_h} (x, a) \left[  \hat \alpha^{\hat \pi}_{h + 1} \right]  - \BB_{\tau_h} (x, a) \left[  \alpha^{\hat \pi}_{h +1 }  \right] \\
	& = H \epsilon(x) + \epsilon(x, a)  + \underbrace{\left( \hat \BB_{\tau_h}(x, a) - \BB_{\tau_h}(x, a) \right) \left[ \hat \alpha^{\hat \pi}_{ h + 1} \right] }_{ \textbf{(I)}}  + \BB_{\tau_h}(x, a) \left[\hat \alpha^{\hat \pi}_{ h + 1} -  \alpha^{\hat \pi}_{h + 1} \right]
	\end{align}
	Next, we use Lemma~\ref{lem::Ibound} to get a bound on \textbf{(I)}.
	
	\begin{restatable}{lemma}{Ibound}\label{lem::Ibound}
		Term \textbf{(I)} is bounded above by the following quantity:
		\begin{align}
		\textbf{(I)} & \leq {  \BB_{\tau_h}(x,a) \left[  \hat \alpha^{\hat \pi}_{h  + 1} - \alpha^{\pistar}_{h  + 1}  \right] \over 2c} + 7 H^2 \tilde \epsilon(c, x,a) + 14 H^2  \tilde \epsilon(1, x, a) + \epsilon(x, a) 
		+ 11 H^2 \sum_{x'} T_\star(x' | x,a) \epsilon(x') 
		\end{align}
	\end{restatable}

	Hence, because $\tilde \epsilon(1, x, a) \leq \tilde \epsilon(c, x, a)$ for $c \geq 1$ and we set $\precConst = \precConstValue$, the expected difference in $\alpha$-vectors is then bounded above by
	\begin{align}
	\E_{\hat \pi} \left[\hat \alpha_{h, \tau_h}^{\hat \pi} (x_h, a_h) - \alpha_{h, \tau_h}^{\hat \pi} (x_h, a_h) \right] & \leq  \E_{\hat \pi} \left[H \epsilon(x_h) + 2\epsilon(x_h, a_h) + \precConst H^2  
	\tilde \epsilon(c, x_h, a_h) + 11 H^2 \sum_{x'} T_\star(x' | x_h, a_h) \epsilon(x') \right]  \\
	& \quad  +   \E_{\hat \pi} \left[ {  \BB_{\tau_h}(x_h, a_h) \left[  \hat \alpha^{\hat \pi}_{h  + 1} - \red{\alpha^{\pistar}_{h  + 1} } \right] \over 2c}   \right]  \\
	& \quad  + \E_{\hat \pi} \left[ \BB_{\tau_h}(x_h, a_h) \left[\hat \alpha^{\hat \pi}_{ h + 1} -  \alpha^{\hat \pi}_{h + 1} \right] \right] \\
	& \leq \E_{\hat \pi} \left[H \epsilon(x_h) + 2\epsilon(x_h, a_h) + \precConst H^2 \tilde \epsilon(c, x_h, a_h) + 11 H^2 \sum_{x'} T_\star(x' | x_h, a_h) \epsilon(x') \right]  \\
	& \quad  +   \E_{\hat \pi} \left[ {  \BB_{\tau_h}(x_h, a_h) \left[  \hat \alpha^{\hat \pi}_{h  + 1} - \red{\alpha^{\hat \pi}_{h  + 1} } \right] \over 2c}   \right]  \\
	& \quad  + \E_{\hat \pi} \left[ \BB_{\tau_h}(x_h, a_h) \left[\hat \alpha^{\hat \pi}_{ h + 1} -  \alpha^{\hat \pi}_{h + 1} \right] \right],
	\end{align}
	where the second line (before-and-after changes marked in red) follows because
	\begin{align}
	\E_{\hat \pi} \left[  \BB_{\tau_h} (x_h, a_h) \left[  \alpha^{\hat \pi}_{h +1 }  \right]  \right] & = \E_{\hat \pi} \left[  V^{\hat \pi}_{h + 1}(\tau_{h + 1}) \right] \\
	& \leq \E_{\hat \pi} \left[  V^{\pistar}_{h + 1}(\tau_{h + 1}) \right] \\
	& = \E_{\hat \pi} \left[  \BB_{\tau_h} (x_h, a_h) \left[  \alpha^{\pistar}_{h +1 }  \right]  \right]
	\end{align}
	since $\pistar$ is the optimal history-dependent policy. Therefore, since $\E_{\hat \pi} \left[  \BB_{\tau_h} (x_h, a_h) \left[ \alpha_{ h+ 1} \right] \right] = \E_{\hat \pi} \left[  \alpha_{h + 1}(x_{h + 1}) \right] $, we conclude that
	\begin{align}
	\E_{\hat \pi} \left[\hat \alpha_{h, \tau_h}^{\hat \pi} (x_h) - \alpha_{h, \tau_h}^{\hat \pi} (x_h) \right] & = \E_{\hat \pi} \left[\hat \alpha_{h, \tau_h}^{\hat \pi} (x_h, a_h) - \alpha_{h, \tau_h}^{\hat \pi} (x_h, a_h) \right] \\
	& \leq \E_{\hat \pi}  \left[  H \epsilon(x_h) + 2 \epsilon(x_h, a_h) + \precConst H^2 \tilde \epsilon(c, x_h, a_h)  \right] \\
	& \quad +  \left( 1 + { 1\over 2c } \right) \E_{\hat \pi} \left[  \hat \alpha^{\hat \pi}_{h + 1, \tau_{h + 1} } (x_{h + 1})  -  \alpha_{h + 1, \tau_{h + 1} }^{\hat \pi} (x_{h + 1})  + 11 H^2 \epsilon(x_{ h+ 1})\right].
	\end{align}
	
\end{proof}

	\subsubsection{Proof of Lemma~\ref{lem::Ibound}}
	Here, we restate the bound on \textbf{(I)} before proving it.
	\Ibound*
	\begin{proof}[Proof of Lemma~\ref{lem::Ibound}]
	\begin{align}
	\textbf{(I)} & = \left( \hat \BB_{\tau_h}(x, a) - \BB_{\tau_h}(x, a) \right) \left[\hat \alpha^{\hat \pi}_{ h + 1} - \alpha^{\pistar}_{ h + 1 }\right] + \left( \hat \BB_{\tau_h}(x, a) - \BB_{\tau_h}(x, a) \right) \left[ \alpha^{\pistar}_{h + 1} \right] \\
	& \leq \left( \hat \BB_{\tau_h}(x, a) - \BB_{\tau_h}(x, a) \right) \left[\hat \alpha^{\hat \pi}_{ h + 1} - \alpha^{\pistar}_{h + 1}\right] + \epsilon(x, a) +  (H - h)\sum_{x'} \hat T(x' | x, a)  \epsilon(x') \\
	& \leq {\BB_{\tau_h}(x, a) \left[\hat \alpha^{\hat \pi}_{ h + 1} - \alpha^{\pistar}_{ h + 1 }\right] \over 2c }   + 6 H^2 \tilde \epsilon(c, x, a)  + \epsilon(x, a) 
	+ 7 H^2 \sum_{x'} \hat  T(x' | x, a)  \epsilon(x') \\
	& \leq {\BB_{\tau_h}(x, a) \left[\hat \alpha^{\hat \pi}_{ h + 1} - \alpha^{\pistar}_{ h + 1}\right] \over 2c }   + 6 H^2 \tilde \epsilon(c, x, a)  + 14 H^2 \tilde \epsilon(1,x, a)  + \epsilon(x, a) 
	+ 11 H^2 \sum_{x'}  T_\star(x' | x, a)  \epsilon(x') .
	\end{align}
	The first inequality uses Lemma~\ref{lem::B-opt-bound} and the second uses Lemma~\ref{lem::B-bound}.  The last inequality applies Corollary~\ref{cor::transition-mean} using $\EE^1_T$, which guarantees that
	\begin{align}
	7H^2 \sum_{x' } \hat T(x' | x, a) \epsilon(x') & \leq   { 11 H^2 } \sum_{x'} T_\star(x' | x, a) \epsilon(x') + 14 H ^2\tilde \epsilon(1, x, a)
	\end{align}
    since $\epsilon(x') \leq 2$ for all $x' \in \XX$ by definition.
	\end{proof}

\subsubsection{Helpers}

\begin{lemma}\label{lem::B-opt-bound}
	If $\EE_T$ and $\EE_O$ hold then,
	\begin{align}\left(\hat \BB_{\tau_h}(x, a) - \BB_{\tau_h}(x, a) \right) \left[\alpha^{\pistar}_{h + 1} \right] \leq \epsilon(x, a) + (H -h) \sum_{x'} \hat T(x' | x, a) \epsilon(x').
	\end{align}
\end{lemma}
\begin{proof}
	We expand the definitions of the $\hat \BB$ and $\BB$ operators and then apply $\EE_T$ and $\EE_O$ directly.
	\begin{align}
	\left(\hat \BB_{\tau_h}(x, a) - \BB_{\tau_h}(x, a) \right) \left[\alpha^{\pistar}_{h + 1}\right] & = \sum_{x', y'}O_\star(y' | x') \left( \hat T(x' | x, a) -  T_\star(x' | x, a)  \right) \left[\alpha^{\pistar}_{h + 1, \tau_h'}(x') \right] \\
	&\quad  +  \sum_{x', y'}  \left( \hat O(y' | x') - O_\star(y' | x')   \right) \hat T(x' | x, a) \left[\alpha^{\pistar}_{h + 1, \tau_h'}(x') \right] \\
	& \leq \epsilon(x, a) + (H - h) \sum_{x'} \hat T(x'| x,a) \epsilon(x').
	\end{align}
\end{proof}

\begin{lemma}\label{lem::B-bound}
	If $ \EE_T^c$ and $\EE_O$ hold, then
	\begin{align}
	\left( \hat \BB_{\tau_h}(x, a) -  \BB_{\tau_h}(x, a) \right)  \left[   \hat \alpha^{\hat \pi}_{h + 1} - \alpha^{\pistar}_{h + 1} \right] & \leq  { \BB_{\tau_h}(x, a) \left[   \hat \alpha^{\hat \pi}_{h + 1} - \alpha^{\pistar}_{h + 1} \right] \over 2c } +  6 H^2 \tilde \epsilon(c, x, a)   \\
	& \quad + 6H^2 \sum_{x'}\hat  T(x' | x, a) \epsilon(x') .
	\end{align}
\end{lemma}
\begin{proof}
	For convenience, let $\alpha := \hat \alpha^{\hat \pi}$ and $\alpha' := \alpha^{\pistar}$. 
	Observe that Proposition~\ref{prop::alpha} and Lemma~\ref{lem::learned-alpha-bound} guarantee that $|\alpha_{h + 1, \tau_h'} - \alpha'_{h + 1, \tau_h'} | \leq 6H (H - h)$. Then,
	\begin{align}
	 & \left( \hat \BB_{\tau_h}(x, a) -  \BB_{\tau_h}(x, a) \right) \left[  \alpha_{h + 1 }  - \alpha'_{h + 1 } \right] \\
	  &= \sum_{x'} \left( \hat T(x' | x, a) - T_\star(x' | x,a) \right) \sum_{y'} O_\star(y' | x')  \left(  \alpha_{h + 1, \tau_h'}(x') -  \alpha'_{h + 1, \tau_h'}(x') \right) \\
	 & \quad + \red{\sum_{x' , y'}  \hat T(x' | x, a) \left(\hat  O(y' | x' ) - O_\star(y' | x')  \right)\left(  \alpha_{h + 1, \tau_h'}(x') -  \alpha'_{h + 1, \tau_h'}(x') \right)} \\
	 & \leq \blue{\sum_{x'} \left( \hat T(x' | x, a) - T_\star(x' | x,a) \right) \sum_{y'} O_\star(y' | x')  \left(  \alpha_{h + 1, \tau_h'}(x') -  \alpha'_{h + 1, \tau_h'}(x') \right)} \\
	 & \quad + \red{6 H(H - h) \sum_{x' , y'}  \hat T(x' | x, a) \epsilon(x') } \\
	 & \leq \blue{{ \sum_{x'} T_\star(x' | x, a) \sum_{y'} O_\star(y' |x' ) \left(  \alpha_{h + 1, \tau_h'}(x') -  \alpha'_{h + 1, \tau_h'}(x') \right) \over 2c }  } \\
	 & \quad \blue{ +  6 H (H - h)  \min\left\{ 2, { 2 c   X \log(X^2 A KH/\delta) \over n(x, a) } \right\}  }  \\
	 & \quad +   6H(H - h) \sum_{x' } \hat T(x' | x, a) \epsilon(x')   \\
	 & \leq  { \BB_{\tau_h}(x, a)  \left[ \alpha_{h + 1} -  \alpha'_{h + 1}  \right]  \over 2c}  
	   + 6 H^2 \tilde \epsilon(c, x, a) 
	  +   6H^2 \sum_{x' } \hat T(x' | x, a) \epsilon(x') .
 	\end{align}
 	The first inequality uses $\EE_O$ to bound the total variation distance between $O_\star$ and $\hat O$ (before-and-after changes marked in red). The second inequality uses $\EE^c_T$ along with Corollary~\ref{cor::transition-mean} by setting $g(x') = \sum_{y'} O(y' | x') \left(  \alpha _{h +1, \tau_h'} (x') - \alpha'_{h + 1, \tau_h'} (x') \right)$ (before-and-after changes marked in blue).   The last inequality simply uses the definition of $\BB$ and $\tilde \epsilon(c, x, a)$ and the fact that $H(H - h) \leq H^2$.

\end{proof}

\subsection{First steps towards function approximation}\label{app::fn-simple}

A natural follow-up question is whether \tabalg can be easily generalized to incorporate function approximation. While we leave in depth discussion of a much more general form of function approximation to Section~\ref{sec::fn}, we remark that it is easy to replace the tabular estimation of $O_\star$ with function approximation as long as the latent states are tabular.

Consider a function class $\Theta \subset (\XX \to \Delta(\YY))$. For simplicity assume that $\Theta$ is finite, in which case we expect the complexity of $\Theta$ to be measured as the log-cardinality $\log(|\Theta|)$, as is standard. Assuming that $\Theta$ realizes $O_\star$ ($O_\star \in \Theta$) and it is proper ($O(\cdot | x) \in \Delta(\YY)$ for all $x \in \XX$ and $O \in \Theta$), one can update $\hat O_k$ via maximum likelihood estimation (MLE):
\begin{align*}
    \hat O_{k + 1}  = \argmax_{O \in \Theta } \sum_{\ell \in [k], h \in [H]} \log O(y^{\ell}_h |x^{\ell}_h ).
\end{align*}
Then, we change the bonuses to be
\begin{align}
    \epsilon_k(x, a) & = \min\left\{2, \sqrt{ C_T' X \log \left(X^2 AK H \right) \over n_k(x, a) } \right\} \\
    \epsilon_k(x) & =\min \left\{ 2,  \sqrt{ C_O' \log (|\Theta| XK/\delta) \over n_k(x) } \right\} ,
\end{align}
where $C_T' = \constantCTnew$ and $C_O' = \constantCOnew$
and the optimistic reward function is changed to
\begin{align}
    \hat r_k(x, a) = r(x, a) + 3 H\epsilon_k(x, a) + H \epsilon_k(x).
\end{align}
Note that the algorithm still requires only point estimates of $O_\star$ as opposed to maintaining a version space. These changes yield the following bound:
\begin{proposition}\label{prop::fn-simple}
Let $\MM$ be a \setshort model with $X$ latent states. With probability at least $1 - 3\delta$, \tabalg with emission function class $\Theta$ outputs a sequence of policies $\hat \pi_1, \ldots, \hat \pi_K$ such that
\begin{align}
    \regret(K) &  =  \OO \left(  \sqrt{  H^5   X  K  \left(  \log (|\Theta|) + \iota \right)   }  + \sqrt{ H^5 X^2 A K \iota  }  \right) 
\end{align}
where $\iota = \log (2 X^2 A KH /\delta) $.
\end{proposition}

We see that the original $Y$ dependence is replaced with the complexity $\log(|\Theta|)$.
An interesting observation of this result is that we do not have to tailor the exploration to the type of function approximator used for $O_\star$ aside from adjustment of the bonus magnitude. The class $\Theta$ can also be completely arbitrary and need not satisfy any further structural conditions besides realizability and learnability for the MLE (i.e. manageable complexity). 

We finally remark that the dependence on $X$ is worse than the purely tabular case. This is due to an alternative technical approach (akin to the difference between the UCRL bound of \citet{auer2008near} and the improved version of \citet{azar2017minimax}). It is possible to apply our original technique to this case as well, but, in contrast to MDPs, this would yield a $\log(Y)$ factor due to a union bound over histories, which is not ideal.  We believe the simplicity of this analysis and ability to handle infinite $Y$ is a more desirable choice.

\subsubsection{Proof of Proposition~\ref{prop::fn-simple}}
We define new high probability events for the count-based estimate $\hat T_k$ and maximum likelihood estimate $\hat O_k$:
\begin{align}
	\EE_T & = \left\{ \forall k \in [K], x \in \XX, a \in \AA, \  \| T(\cdot | x, a)  - \hat T_k(\cdot | x, a) \|_1 \leq \sqrt{  C_T' X \log(X^2AK H / \delta) \over n_k(x, a)}  \right\} \\
	\EE_O & = \left\{  \forall k \in [K], x \in \XX, \ \| O(\cdot | x) - \hat O_k(\cdot | x) \|_1 \leq \sqrt{C_O'  \log (X K  |\Theta|  /\delta )  \over n_k(x)}  \right\}.
\end{align}
with $C_T' = \constantCTnew$ and $C_O' = \constantCOnew$.
$P(\EE_T) \geq 1 - \delta$ follows the same proof as Lemma~\ref{lem::emission-tab} and $P(\EE_O)\geq 1 - \delta$ follows the same proof as Lemma~\ref{lem::mle-concentration} but applied to the emission function (see also Theorem~21 of \citet{agarwal2020flambe}). This guarantees that, with probability at least $1 - \delta$,
\begin{align}
    n_k(x) \| O(\cdot | x^\ell_h) - \hat O_k(\cdot | x^\ell_h) \|^2_1 & \leq  \sum_{\ell \in [k - 1], h \in [H]} \| O(\cdot | x^\ell_h) - \hat O_k(\cdot | x^\ell_h) \|_1^2 \\
    & \leq  8 \log (K |\Theta| /\delta)
\end{align}
for all $k \in [K]$. Rearranging ensures the claim on $P(\EE_O)$. Assuming these events hold, we show that this ensures optimism of the $\alpha$-vectors as before. 

 Let $\hat \alpha$ denote the $\alpha$-vector for the estimated model $\hat T$ and $\hat O$ and let $\alpha$ be the one for the true model.

\begin{lemma}\label{lem::alpha-optimism-simple} Let the above events hold. Then,
    $ \hat \alpha^\pi_{k, h, \tau_h } \geq \alpha^\pi_{h, \tau_h} (x) + H \epsilon_k(x)$
\end{lemma}
\begin{proof}
For now, we omit the subscript notation denoting the round $k$. By the above events, we have that
\begin{align}
		\hat \alpha^\pi_{H, \tau_H}(x) & = \sum_{a} \pi(a | \tau_h) \left(  r(x, a) + H \epsilon(x) + 3 H\epsilon(x, a) \right),
	\end{align}
	which means that
	\begin{align}
		\hat \alpha^\pi_{H, \tau_H}(x) -  \alpha^\pi_{H, \tau_H}(x) & \geq  H \epsilon(x) +  3H \epsilon(x, a) \\
		& \geq H \epsilon(x)
	\end{align}
	Inductively, assume that $\hat \alpha^\pi_{ h + 1, \tau_{ h+ 1}}(x) \geq \alpha^\pi_{ h + 1, \tau_{ h+ 1}}(x) +   H \epsilon(x) $. Then, using recursive definitions of the $\alpha$-vectors, 
	\begin{align}
		\hat \alpha^\pi_{h, \tau_h} (x,a) -  \alpha^\pi_{h, \tau_h} (x,a) & = \hat r(x, a) +  \sum_{x',  y'} \hat T(x' | x, a) \hat O(y' |  x') \hat \alpha^\pi_{h + 1, \tau_h'}(x')  \\
		& \quad - r(x, a) - \sum_{x',  y'} T_\star(x' | x, a) O_\star(y' |  x') \alpha^\pi_{h + 1, \tau'_h}(x') \\
		& \geq \hat r(x, a) - r(x, a) \\
		& \quad  + \sum_{x',  y'} \hat T(x' | x, a) \hat O(y' |  x')  \left( \alpha^\pi_{h + 1, \tau_h'}(x') +  H \epsilon(x') \right) \\
		& \quad - \sum_{x',  y'} T_\star(x' | x, a) O_\star(y' |  x') \alpha^\pi_{h + 1, \tau'_h}(x') \\
		& =  \hat r(x, a) - r(x, a) \\
		& \quad  + \sum_{x',  y'} \left(\hat T(x' | x, a) - T_\star(x'  | x, a)  \right)  \hat O(y' |  x') \alpha^\pi_{h  + 1, \tau_{h}'}  (x') \\
		& \quad + \sum_{x', y'} T_\star(x' | x, a) (x') \alpha^\pi_{h  + 1, \tau_{h}'}  \left(  \hat O(y' | x') - O_\star(y'  | x')  \right)  \\
		& \quad +  H \sum_{x', y'} \hat T(x' | x, a) \hat O(y' | x')  \epsilon(x'),
	\end{align}
	where $\tau_h'$ is again the concatenation of $\tau_h$ with $y'$ and $a$. Now, we can lower bound the above using the total variation distance:
	\begin{align}
		\hat \alpha^\pi_{h, \tau_h} (x,a) -  \alpha^\pi_{h, \tau_h} (x,a) & \geq \hat r(x, a) - r(x, a) \\
		& \quad - H  \| \hat T(\cdot | x, a) - T_\star(\cdot | x, a)  \|_1  \\
		&\quad - H \sum_{x'} T_\star(x' | x, a) \| \hat O(\cdot | x') - O_\star(\cdot | x') \|_1  \\
		& \quad +  H \sum_{x'} \hat T(x' | x, a)  \epsilon(x'). \\
	\end{align}
	Recognizing that $\epsilon(x') \leq 2$ by definition, the last term is lower bounded as
	\begin{align}
		H \sum_{x'} \hat T(x' | x, a)  \epsilon(x') & \geq  - 2 H \| \hat T(\cdot | x, a)  - T_\star(\cdot | x, a)  \|_1  + H \sum_{x'} T_\star(x' | x, a)  \epsilon(x')
	\end{align}
	Putting these all together, we have
	\begin{align}
		\hat \alpha^\pi_{h, \tau_h} (x,a) -  \alpha^\pi_{h, \tau_h} (x,a) & \geq \hat r(x, a) - r(x, a) \\
		& \quad  - 3 H \| \hat T(\cdot | x, a) - T_\star(\cdot | x, a) \|_1 \\
		& \quad -  H \sum_{x'} T_\star(x' | x, a) \left(  \| \hat O(\cdot | x') - O_\star(\cdot | x') \|_1 - \epsilon(x')  \right) \\
		& \geq \hat r(x, a) - r(x, a) \\
		& \quad  - 3 H \| \hat T(\cdot | x, a) - T_\star(\cdot | x, a) \|_1 \\
		& \geq  H \epsilon(x) + 3 H \epsilon(x, a) - 3 H\| \hat T(\cdot | x, a) - T_\star(\cdot | x, a) \|_1 \\
		& \geq H \epsilon(x) 
	\end{align}
	Applying this inductive argument backwards along $h = H, \ldots, 1$ gives the result.
	
\end{proof}

\begin{lemma}
	Let the above events hold. Then, for any history-dependent policy $\pi$, it holds that $\hat v_k(\pi) - v(\pi) \geq 0 $, where $\hat v_k$ is the policy value under the model $\hat T_k$, $\hat O_k$, and $\hat r_k$.
\end{lemma}
\begin{proof}
	The proof is immediate from the $\alpha$-vector representation of the value functions and Lemma~\ref{lem::alpha-optimism-simple}:
	\begin{align}
\hat v_k(\pi) - v(\pi) & = \sum_{x_1, y_1}  \hat O(y_1 | x_1) \rho(x_1) \hat \alpha^\pi_{k, 1, \tau_1} (x_1)  -  \sum_{x_1, y_1}  O_\star(y_1 | x_1) \rho(x_1)  \alpha^\pi_{1, \tau_1} (x_1) \\
& = \sum_{x_1, y_1} \left( \hat O_k(y_1 | x_1) - O_\star(y_1 | x_1) \right) \rho(x_1) \alpha^{\pi}_{1, \tau_1}(x_1) \\
&\quad  + \sum_{x_1, y_1}  \hat O(y_1 | x_1) \rho(x_1) \left( \hat \alpha^\pi_{k, 1, \tau_1} (x_1) - \alpha^\pi_{1, \tau_1}(x_1)  \right)    \\
& \geq   - \sum_{x_1} H \rho(x_1) \epsilon_k(x_1)   +   \sum_{x_1, y_1} H \hat O(y_1 | x_1) \rho(x_1) \epsilon_k(x_1) \\
& = 0.
	\end{align}
\end{proof}

We are now ready to prove the result. Let $\hat \E_{k}$ and $\hat P_k$ denote the expectation and measure under the learned model at round $k$. Then,
\begin{align}
	\regret(K) & = \sum_{k \in [K]} v(\pistar) - v(\hat \pi_k) \\
	& \leq \sum_{k \in [K]} \hat v_k(\hat \pi_k) - v(\hat \pi_k) \\
	& = \sum_{k \in [K]}   \hat \E_{k, \hat \pi_k}  \left[   \ \sum_{h \in [H]}  \hat r_k(x_h, a_h) \right] - \E_{\hat \pi_k} \left[  \sum_{h \in [H]}   r(x_h, a_h)  \right] \\
	& = \sum_{k \in [K]}   \hat \E_{k, \hat \pi_k} \left[   \ \sum_{h \in [H]}   r(x_h, a_h) + H\epsilon_k(x_h) +  3H \epsilon_k(x_h, a_h) \right] - \E_{\hat \pi_k} \left[  \sum_{h \in [H]}   r(x_h, a_h)  \right] \\
	& \leq \sum_{k \in [K]}  H \| \hat P_{k, \hat \pi_k} - P_{\hat \pi_k} \|_1  + 3 H \hat \E_{k, \hat \pi_k} \left[  \sum_{h \in [H]} \epsilon_k(x_h) + \epsilon_k(x_h, a_h) \right].
\end{align}
Since $\epsilon_k(x)$ and $\epsilon_k(x, a)$ are no greater than $2$, we can change distributions in the last term of the previous display to get
\begin{align}
	\regret(K) & \leq  \sum_{k \in [K]}  13 H^2 \| \hat P_{k, \hat \pi_k} - P_{\hat \pi_k} \|_1 + 3 H \E_{\hat \pi_k}  \left[  \sum_{h \in [H]} \epsilon_k(x_h) + \epsilon_k(x_h, a_h)  \right].
\end{align}
To bound the remaining terms, we rely on the Azuma-Hoeffding inequality (Lemma~\ref{lem::azuma}) with probability at least $1-  \delta$ and the simulation lemma (Lemma~\ref{lem::sim-lem}). 

The regret can then further be bounded as
\begin{align}
	\regret(K) & \leq \sum_{k \in [K]}  13 H^2  \E_{\hat \pi_k} \left[  \sum_{h \in [H]} \epsilon_k(x_h) + \epsilon_k(x_h, a_h) \right] + 3 H \E_{\hat \pi_k}  \left[  \sum_{h \in [H]} \epsilon_k(x_h) + \epsilon_k(x_h, a_h)  \right] \\
	& = 16H^2  \sum_{k \in [K], h \in [H] }  \E_{\hat \pi_k}  \left[ \epsilon_k(x_h) + \epsilon_k(x_h, a_h) \right] \\
	& \leq 64H^2 \sqrt{ KH \log(2/\delta)} +   16 H^2 \sum_{k, h} \epsilon_{k}(x_h^k) + \epsilon_k(x_h^k, a_h^k) \\
	& \leq 64H^2 \sqrt{ K  H \log(2/\delta)} +   16 H^2 \sum_{k, h} \sqrt{ C_O'  ( \log (|\Theta |) +  \iota )  \over \max\{ 1, n_k(x_h^k)\}} + \sqrt{ C_T' X \iota \over \max\{ 1, n_k(x_h^k, a_h^k) \} }  \label{eq::conversion},
\end{align}
To bound this final term, as before, we appeal to pigeonhole principle (Lemma~\ref{lem::pigeon}):
	\begin{align}
	\regret(K) & \lesssim  \sqrt{  H^5 K  \log(2/\delta)} +    \sqrt{  H^5  X  K  \left( \log (|\Theta|) + \iota \right)  }  + \sqrt{ H^5 X^2 A K \iota  } \\
	&\quad + H^3 X \sqrt{ \log (|\Theta|) +  \iota   } +  H^3 XA \sqrt{ X \iota }.
\end{align}
By a union bound on the events $\EE_T$ and $\EE_O$ and the Azuma-Hoeffding event, we conclude that this occurs with probability at least $1-  3\delta$.

\section{Proof of  Lower Bound Theorem~\ref{thm::lower-bound}}\label{app::lb}

\begin{figure}
	\centering
	\includegraphics[width=5in]{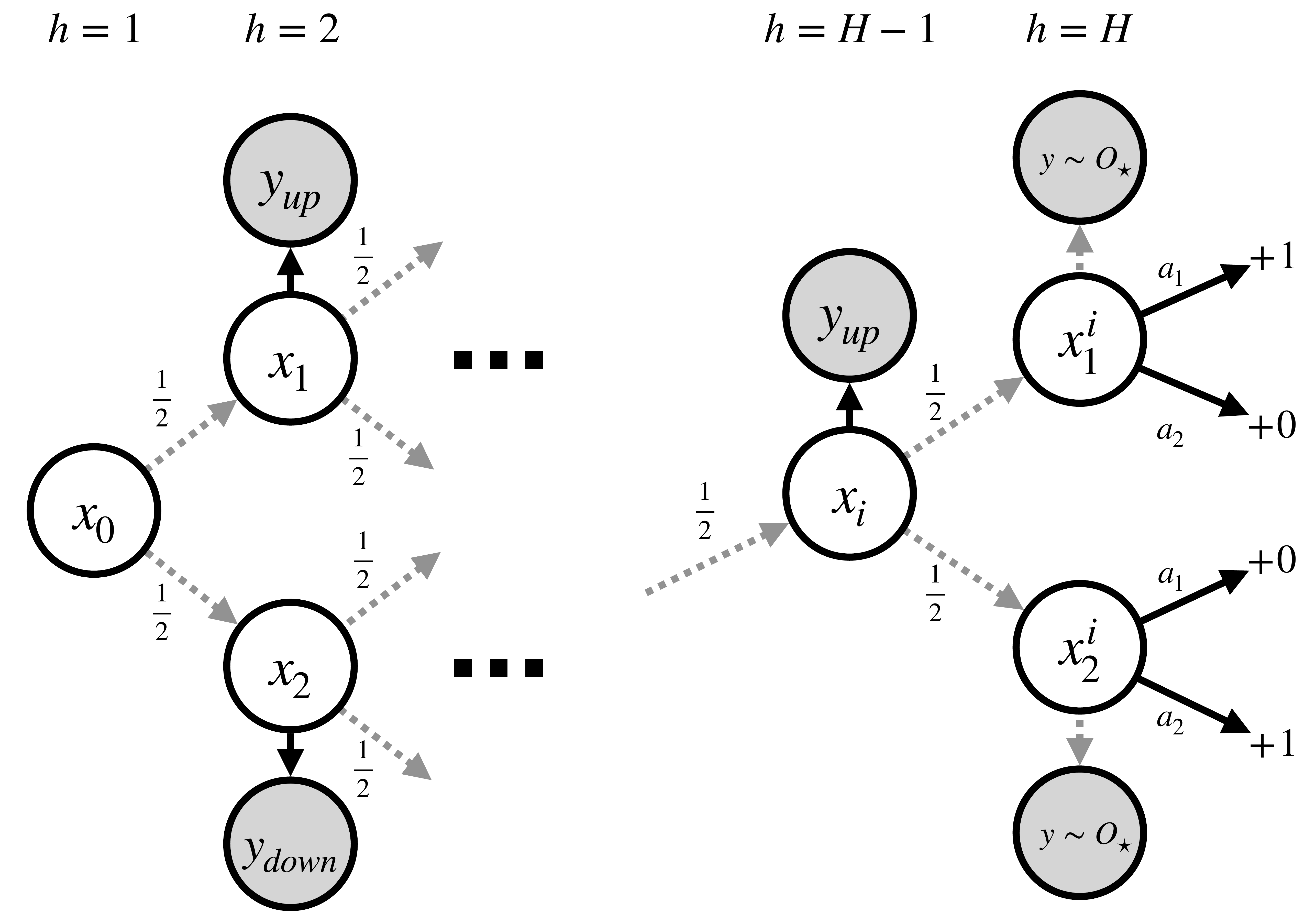}
	\caption{Hard instance of POMDP. The latent state space is a binary tree starting at $x_0$ and the learner traverses a layer each time step. The root and and second layer depicted on the left. Observations about which direction it traversed are revealed at each step. In the last and second to last layer, one of many triplets like Figure~\ref{fig::lbsmall} is encountered (only one depicted in this figure on the right). The policy has no control until the last layer. }\label{fig::instance}
\end{figure}

In this section, we formally prove the information-theoretic lower bound of Theorem~\ref{thm::lower-bound}. We take the standard minimax approach: design a class of problem instances $\UU$ such that, for any algorithm that generates a policy $\hat \pi$ over $K$ episodes of interaction, there is always a problem instance $u \in \UU$ with
\begin{align*}
    \E_{u} \left[ v(\pistar) - v(\hat \pi)\right] \geq \epsilon
\end{align*}
for some $\epsilon$ that we will try to control. The expectation is over the algorithm and data generated under the instance $u$.

While the intuition for the lower bound in the main paper with $X = 3$ is useful for understanding how the $Y$ dependence appears in the lower bound, it is not immediate to generalize. Note that it is not sufficient to take the naive approach of simply maintaining $H = 1$ and increasing the number of latent states. In such cases, the possible problem instances will either have very low loss separation or be very easy to test and differentiate. This would mean an algorithm could easily either find out how to act optimally or not acting optimally is ``close enough." The issue is that the posterior distribution over $\XX$ given an observation would end up being very close to uniform (to avoid being able to simply test to distinguish between instances). However, this allows for a potentially large margin of error for any policy since even the optimal policy will struggle greatly to achieve high rewards regardless.

This issue can be resolved by leveraging the history to reduce the problem to solving many (on the order $X$) $3$-latent state subproblems. Consider the sketch POMDP in Figure~\ref{fig::instance}, which is the full version of the on the one in Figure~\ref{fig::lbsmall}. Figure~\ref{fig::instance} depicts a binary tree in which the policy randomly traverses the nodes to the leaves of the last layer $H$. The last and second to last layers collectively make up a swath of $3$-latent state problems. The catch is that, as the policy traverses the nodes, the observations reveal its exact state in the tree up to layer $H - 1$ by revealing whether it traversed to the upper or lower child at each step with $\{\yup, \ydown\}$. Thus it can exactly decode its state at layer $H - 1$ and find out which of the $3$-latent state problems it is in. At layer $H - 1$, it is faced with one of the possible $3$-latent state problems and it must act optimally.

We will design the emission function such that it is difficult to decode whether the learner is in the upper or lower child in the same way as we described for the case when $X = 3$.
Thus, the learner will have to visit a given $h = H - 1$ layer parent at least $ \Omega(Y)$ times before learning the optimal policy for that parent. Furthermore, it must learn the optimal policy for at least a constant fraction of the parents (size $ \Omega(X)$) to compete with the full optimal policy. Together this yields total interactions on the order of $\Omega(XY)$.

\lowerbound*

\subsection{Construction of instance class}

Note that the preconditions ensure that $Y(X + 1) \geq 512 \log 2$.

The learner starts deterministically at $x_0$. Without loss of generality, we assume that $X + 1$ is a power of $2$ and $Y$ is even. Otherwise, we can reduce $X + 1$ to the nearest power of $2$ and reduce $Y$ to the nearest even number by at most losing a constant factor in the bound.  There are a total of $H = \log_2(X + 1)$ timesteps. Before the $H$th timestep, rewards are all zero and the learner transitions uniformly randomly from the current latent state to either the upper or lower child latent state in the next layer, regardless of the action (see Figure~\ref{fig::instance}).
An observation $\yup$ or $\ydown$ is revealed indicating whether the learner is currently in the upper or lower child of the parent latent state. Hence, for $h \leq H - 1$, the latent state can be exactly decoded given the history of observations. We assume that all states in the final layer $h = H$ then transition to a dummy state such as $x_{0}$ at $H + 1$.

We denote the states of the final layer $h  =H$ by $\XX'$. Note that $\XX'$ consists of $X' := { X + 1 \over 2}$ latent states. States in $\XX'$ can be grouped into ${X' \over 2}$ groups of $2$ where the states in a group share the same parent. Let $x^1_i$ and $x^2_i$ be the upper and lower children of the same parent state $x_i$, respectively. The reward function is defined as
\begin{align}
 r(x^1_i, a)  = \begin{cases}
 1 & a = a_1 \\
 0 & a = a_2
 \end{cases}
\end{align}
\begin{align}
r(x^2_i, a)  = \begin{cases}
0 & a = a_1 \\
1 & a = a_2
\end{cases}
\end{align}
This is duplicated for all the parent-child triplets $(x_i, x_i^1, x_i^2)$ in the second to last and last layers. The emission function in the final layer is different depending on the parent $x_i$ and the child. For a child latent state $x \in \XX'$, we design $O_\star(\cdot |x)$ to be supported on $\YY' := \YY \setminus \{ \yup, \ydown\}$ with cardinality $Y' = \abs{\YY'}$, which will be specified presently. The instances will vary based on the selection of $O_\star$.

\subsection{Selection of emission function}

 We construct the instances by perturbing the probabilities in $O_\star$ so that they deviate slightly from uniform. To ensure that probabilities properly sum to $1$, we will split $\YY'$ into equal partitions $\YY'_+$ and $\YY'_-$ each of size $|Y'|\over 2$. We can construct a bijection such that for any $y \in \YY'_-$ there is a ``mirror" $y_+$ in $\YY'_+$. An instance in the class will be specified by some vector $u \in \{-1, 1\}^{ {X'Y' \over 4}} $ which determines the observation matrix. We will denote the observation matrix for instance $u$ with $O_{\star, u}$. Fix $\epsilon > 0$. We index into the vector $u$ with an observation in $y \in \YY'_+$ and a parent latent state $x_i \in \XX$ in the second to last layer. Let $x^1_i$ and $x^2_i$ be leaf latent states that share the same parent $x_i$. Note that this is valid since the child states (as we construct them) are distinct for each parent.  Then for $x^1_i$, define 
\begin{align}
O_{\star, u} (y | x^1_i) = { 1 + u(y, x_i) \epsilon \over  Y'  } \quad \forall y \in \YY'_+,\\
O_{\star, u}(y | x^1_i) = { 1 - u(y_+, x_i) \epsilon \over  Y'  } \quad \forall y \in \YY'_-, 
\end{align}
and for $x_2$, 
\begin{align}
O_{\star, u}(y | x^2_i) = { 1 - u(y, x_i) \epsilon \over  Y'  } \quad \forall y \in \YY'_+,\\
O_{\star, u}(y | x^2_i) = { 1 + u(y_+, x_i) \epsilon \over  Y'  } \quad \forall y \in \YY'_-.
\end{align}
We will also use $P_u$ and $\E_u$ to denote the measure and expectation, respectively, under instance $u$.
We can verify that, conditioned on the the parent $x_i$, the distribution over $\YY'$ is uniform:
\begin{align}
P_u(y | x_i)  & = P(x_i^1 | x_i)O_{\star, u}(y | x^1_i) + P(x^2_i | x_i) O_{\star, u}(y | x^2_i)  \\
& =  { O_{\star, u}(y | x^1_i) + O_{\star, u}(y | x^2_i) \over 2 } = {1 \over Y'}.
\end{align}
It is also worth noting that, conditioned on the parent $x_i$ (equivalently, on the history $y_{1:H - 1}$), the posterior is:
\begin{align}
P_u(x^1_i  | y_H, x_i) & = { 1 + u(y_H, x_i) \epsilon  \over 2} \quad \forall y \in \YY'_+ \\
P_u(x^1_i  | y_H, x_i) & = { 1 - u(  (y_H)_+, x_i) \epsilon \over 2 } \quad \forall y \in \YY'_-.
\end{align}

Furthermore, by Lemma 4.7 of \citet{massart2007concentration}, 
there exists $\UU \in \{-1, 1\}^{X'Y'/4}$ such that $\abs{\UU} \geq \exp(Y'X'/32)$ and $\| u - u'\|_1 \geq { X' Y' \over 8}$ for all $u, u' \in \UU$ such  that $u \neq u'$.

\subsection{Separability condition} 
We now show in this construction that no history-dependent policy can perform well on all instances in $\UU$ simultaneously.
Let $\Pi$ be the class of all history-dependent, deterministic policies. Let $u \in \UU$ be fixed. It is clear that the optimal policy for instance $u$ chooses action $\hat a$ such that $r(\hat x, \hat a) = 1$ where $\hat x = \argmax_{x} P_u(x | y_{1:H}) \in \XX'$ maximizes the posterior. We denote this policy by $\pistar_{u}$. 

Consider an arbitrary $\pi$. Recall that, by construction, there is a bijection between $y_{1:H - 1}$ and the parent latent states at layer $H - 1$.
To avoid notational clutter, we will now denote this by $z$ (which we have previously written as $x_i$). Thus, we can equivalently write $\pi$ as a function of $z$ and the last observation $y_H$. Define
\begin{align}
v_{u}(\pi | z) = \E_u \left[ r(x, \pi(y_H, z)  ) \ | \ z \right]
\end{align}
as the conditional value of $\pi$ given it has reached the parent $z$ in instance $u$. A straightforward calculation shows that 
\begin{align}
v_u(\pistar_u | z)  - v_u(\pi | z) = \epsilon \cdot {N_z(\pi, \pistar_u) \over Y'} 
\end{align}
where $N_z (\pi, \pistar_u) := \sum_{y \in \YY'} \1\left\{  \pistar_u(y, z) \neq \pi(y, z)\right\}$ is the number of observations on which $\pi$ and $\pistar_u$ disagree given parent $z$. The sub-optimality gap on the full instance is simply the average over the ${X' \over 2}$ parents:
\begin{align}
v_u(\pistar_u) - v_u(\pi)  = \epsilon \sum_{z} { 2 N_z(\pi, \pistar_u) \over Y' X'} =  { 2\epsilon N(\pi, \pi_u^*)\over Y' X' },
\end{align}
where $N(\pi, \pi_u^*) = \sum_z \sum_{y \in \YY'} \1\left\{  \pistar_u(y, z) \neq \pi(y, z)\right\}$ is the total number of disagreements at layer $H$. Then, for any $u, u' \in \UU$ such that $u \neq u'$,
\begin{align}
\underbrace{v_u(\pistar_u) - v_u(\pi)}_{\text{error on instance } u } + \underbrace{v_{u'}(\pistar_{u'}) - v_{u'}(\pi)}_{\text{error on instance } u'} & = { 2 \epsilon \over Y' X' }  \left(  N(\pi, \pistar_u) + N(\pi, \pistar_{u'}) \right) \\
& \geq { 2 \epsilon N(\pistar_u, \pistar_{u'}) \over Y' X' }.
\end{align}
Finally, we recall that $\UU$ is such that $\|u - u'\|_1 \geq { X' Y' \over 8}$, which ensures that $u$ and $u'$ differ on at least ${ X' Y' \over 16}$ elements, which implies that $N(\pistar_u, \pistar_{u'}) \geq { X' Y' \over 16}$. Thus, we have
\begin{align}
v_u(\pistar_u) - v_u(\pi) + v_{u'}(\pistar_{u'}) - v_{u'}(\pi) & \geq { \epsilon \over 8}
\end{align}

\subsection{Fano's inequality application}

Thus far, we have detailed the instance class and shown that no policy can achieve error less than $\epsilon \over 16$ on more than one instance. To complete the proof, we apply Fano's inequality to show that these instances are essentially indistinguishable. 

Let $\hat \pi$ be the output of any algorithm $\AF$ that samples from a POMDP instance over $K$ episodes (which a random variable dependent on the instance in which it is run). We have 
\begin{align}
\max_{u \in \UU} \E_u\left[ v_u(\pistar_u) - v_u(\hat \pi) \right] & \geq { \epsilon \over 16 } \inf_{\Psi} {1 \over \abs{\UU}} \sum_{u \in \UU}  P_u\left( \Psi \neq u \right)
\end{align}
where $\Psi$ is a data-dependent test function, the $\inf$ is taken over all measurable tests, $P_u$ denotes the measure under instance $u$. By Fano's inequality,
\begin{align}
\max_{u \in \UU} \E_u\left[ v_u(\pistar_u) - v_u(\hat \pi) \right] & \geq  { \epsilon \over 16}  \left(  1 - {\max_{u \neq u'}  D_{KL}\left(P_u \| P_{u'}   \right)  + \log 2  \over \log |\UU|  }\right)
\end{align}

where $P_u$ and $P_{u'}$ are measures under instances $u$ and $u'$, respectively. Note that these are dependent on the algorithm $\AF$, which determines which actions to take over the $K$ episodes. That is, the probability of taking action $a_h^k$ in round $k$ at step $h$ is $\AF(a^k_h  | \bar \tau^{1:k - 1},  \bar \tau^k_h )$ where we recall that $\bar \tau_h^k = (x_{1:h}, y_{1:h}, a_{1: h - 1})$ is the partial trajectory and $\bar \tau^k$ is the full trajectory, both containing the latent states.  Crucially, note that we allow $\AF$ to be dependent on the latent states. Note that, for \setshort model, we usually assume that this is further reduced to only dependence on historical observations and actions within the current partial trajectory so that $\bar \tau_{h}^k$ becomes $\tau_{h}^k$ in the conditional part. However, this generality allows us to capture algorithms that also can access the underlying state even during training deployments.

The chain rule of the KL divergence gives us the following decomposition in terms of conditional KL divergences:
\begin{align*}
    D_{KL}(P_u \| P_{u'} )  & = \sum_{k \in [K]} \E_{\bar \tau^{1:k - 1}} \left[ D_{KL}\left( P_u(\bar \tau^k ) \| P_{u'} (\bar \tau^k ) \   | \ \bar \tau^{1:k - 1} \right)  \right].
\end{align*}
where we abuse notation slightly and let $P_u(\bar \tau^k )$ denote the distribution over trajectory $\bar \tau^k$. Then, the individual terms are also written as
\begin{align}
      \E_{\bar \tau^{1:k - 1}} \left[ D_{KL}\left( P_u(\bar \tau^k ) \| P_{u'} (\bar \tau^k )  \   | \ \bar \tau^{1:k - 1}  \right)  \right] & =  \E_{\bar \tau^{1:k - 1}} \left[ \sum_{\bar \tau^k} P_u(\bar \tau^k \ | \  \bar \tau^{1:k - 1} ) \log \frac{P_u(\bar \tau^k \  | \  \bar \tau^{1:k - 1} )}{P_{u'}(\bar \tau^k \ | \ \bar \tau^{1:k - 1} )}\right] 
      .
\end{align}
Observe that the conditional probability of a trajectory is
\begin{align}
        P_u(\bar \tau^k \  | \  \bar \tau^{1:k - 1} ) & = P_u(x_1^k)
        \prod_{h = 1}^{H} O_{\star, u} (y_{h}^k  | x_h^k) \AF(a_h | \bar \tau_h^k, \bar \tau^{1:k - 1})  T_{\star, u}(x_{h+1}^k | x_h^k, a_h^k)  
\end{align}
Between the instances $u$ and $u'$, everything is the same (including $\AF$) except for $O_{\star, u}(\cdot | x^k_H)$ and $O_{\star, u'}(\cdot | x^k_H)$ in the last layer by construction. Furthermore, we have that $P(x^k_H | \bar \tau^{1:k - 1}) = P(x^k_H) = {1 \over X'}$ since the policy has no control over the first $H - 1$ steps. Therefore, the conditional KL divergence for any $k$ becomes:
\begin{align}
    \E_{\bar \tau^{1:k - 1}} \left[ D_{KL}\left( P_u(\bar \tau^k ) \| P_{u'} (\bar \tau^k )  \   | \ \bar \tau^{1:k - 1}  \right)  \right] & = \sum_{y_H, x_H} P_u( x_H | \bar \tau^{1:k - 1}_H )  \log {O_{\star, u}(y_H | x_H) \over O_{\star, u'}(y_H | x_H)} \\
    & = {1 \over X'} \sum_{x_H \in \XX'} \sum_{y_H} \log {O_{\star, u}(y_H | x_H) \over O_{\star, u'}(y_H | x_H)} \\  
    & = {1 \over X'} \sum_{x \in \XX'}  \sum_{y \in \YY'_-}  O_{\star, u}(y | x) \log {O_{\star, u}(y | x) \over O_{\star, u'}(y | x)} +  O_{\star, u}(y_+ | x) \log {O_{\star, u}(y_+ | x) \over O_{\star, u'}(y_+ | x)} \\
    & \leq { 2 \over X' Y' } \sum_{x \in \XX', y \in \YY'_-} (1 + \epsilon) \log { 1+ \epsilon \over 1-  \epsilon} + (1 - \epsilon)\log { 1- \epsilon \over 1+  \epsilon} \\
    & \leq { 2 \over X' Y' } \sum_{x \in \XX', y \in \YY'_-} 8 \epsilon^2 \\
    & = 8 \epsilon^2
\end{align}

for $\epsilon < {1 /2}$. Therefore, for if $K \leq { \log |\UU | - 2\log 2 \over 16 \epsilon^2 } $, we have
\begin{align}
    \max_{u \in \UU} \E_u\left[ v_u(\pistar_u) - v_u(\hat \pi) \right] & \geq  { \epsilon \over 16}  \left(  1 - { 8K \epsilon^2    + \log 2  \over \log |\UU|  }\right)  \\
    & \geq { \epsilon \over 32 } 
\end{align}
From the lower bound on the size of $\UU$,  we have that
\begin{align}
    K \leq {X'  Y'   \over 1024 \epsilon^2 }
\end{align}
implies the above condition because
\begin{align}
    K  &\leq {X'  Y'   \over 1024 \epsilon^2 } \\
    & \leq { X' Y' /32 - 2\log 2 \over 16 \epsilon^2 } \\
     & \leq { \log | \UU | - 2\log 2 \over 16 \epsilon^2 } 
\end{align}
as long as $X' Y' \geq 64 \cdot 2 \log 2$. Since $Y' = Y - 2$ and $X' = {X + 1\over 2}$, we have
\begin{align*}
     X' Y' = {(X + 1) (Y - 2) \over 2} \geq  { (X + 1) Y \over 4 } \geq 128\log 2
\end{align*}
where the second inequality follows from the constraints on $X$ and $Y \geq 6$.

\section{Proof of Theorem~\ref{thm::fn-approx}}

A core component of the analysis is a simulation lemma that bounds the difference in values of a policy on two different POMDPs via the total variation distance of their models.
\begin{restatable}[Simulation Lemma]{lemma}{simulation}
	\label{lem::sim-lem}
	Consider a POMDP model with transition matrix $\hat T$, emission matrix $\hat O$, and reward function $r$. Denote the value function and measure under this POMDP by $\hat v$ and $\hat P$ respectively. Then, for any history-dependent policy $\pi$, 
	\begin{align*}
	    \| \hat P_\pi - P_\pi\|_1 & \leq \E_\pi  \sum_{h \in [H] } \| O_\star(\cdot | x_h) - \hat O(\cdot | x_h) \|_1  + \| T_\star(\cdot | x_h, a_h) - \hat T(\cdot | x_h, a_h) \|_1.
	\end{align*}
	Furthermore,
	\begin{align*}
		 |v(\pi) - \hat v(\pi) |  & \leq   H \E_\pi  \sum_{h \in [H] } \| O_\star(\cdot | x_h) - \hat O(\cdot | x_h) \|_1 + \| T_\star(\cdot | x_h, a_h) - \hat T(\cdot | x_h, a_h) \|_1,
	\end{align*}
	where $\E_\pi$ denotes the expectation following policy $\pi$ under the true model $\MM$.
\end{restatable}
This observation, although crude sometimes,\footnote{Indeed, jumping straight to total variation bounds can lead to worse sample complexity bounds in the tabular case, which is why we opt for a more refined $\alpha$-vector analysis for Theorem~\ref{thm::tabular}} is useful in this setting because it is possible to estimate the models in \setshort, in contrast to the POMDP where faithful recovery of $O_\star$ and $T_\star$ is not always possible. 

\subsection{High-probability events}

Similar to the tabular setting, we define the following events and later show that they each occur with high probability.

\begin{align*}
\EE_\TT & = \left\{
\forall k \in [K'], T \in \TT , \ \sum_{\ell \in [k], h \in [H]}  \| T(\cdot  |x_h^\ell, a_h^\ell) -   T_\star(\cdot |x_h^\ell, a_h^\ell ) \|^2_1  \leq  8 \log(K' |\TT|  /\delta)  + 4 \sum_{\ell, h} \log  { T_\star( \tilde x^\ell_h | x^\ell_h, a^\ell_h) \over   T(\tilde x^\ell_h | x^\ell_h, a^\ell_h) }  
\right\} \\
\EE_\Theta  &  = \left\{
\forall k \in [K'], O \in \Theta , \ \sum_{\ell \in [k], h \in [H]}  \| O(\cdot  |x_h^\ell) -   O_\star(\cdot |x_h^\ell ) \|^2_1  \leq  8 \log(K' |\Theta|  /\delta)  + 4 \sum_{\ell, h} \log  { O_\star( y^\ell_h | x^\ell_h) \over   O(y^\ell_h | x^\ell_h) }  
\right\} 
\end{align*}
The intersection of the above two is defined as $\EE_{\TT, \Theta} = \EE_\TT\cap \EE_\Theta$. 
Finally,
let $\EE_{\text{Fre}}$ denote the event that,for all $k \in [K']$ and $h \in [H]$, with $\tilde \pi_\ell  =\hat \pi_\ell \circ_h \unif(\AA)$,
\begin{align}
     & \sum_{\ell \in [k - 1]} \E_{\tilde \pi_\ell} \left[ \| T_\star(\cdot  |x_{ h}) - \hat T_{k}(\cdot  | x_{ h}, a_h) \|_1^2  +  \| O_\star(\cdot  |x_{ h}, a_h) - \hat O_{k} (\cdot  |x_{ h}) \|_1^2 \right] \\ 
    & \leq  32 \log \left( K'  H | \TT \times \Theta | /\delta  \right) +  2\sum_{\ell \in [k - 1]}   \| T_\star(\cdot  |x_{ h}^\ell, a_h^\ell) - \hat T_k(\cdot  |x_{ h}^\ell, a_h^\ell) \|_1^2  +  \| O_\star(\cdot  |x_{ h}^\ell) - \hat O_k(\cdot  |x_{ h}^\ell) \|_1^2.
\end{align}

\begin{lemma}\label{lem::mle-concentration}
    $P(\EE_\TT) \geq 1 - \delta$.
\end{lemma}
\begin{proof}
    See Appendix~\ref{app::mle-concentration-proof}.
\end{proof}
\begin{lemma}
    $P(\EE_\Theta) \geq 1 - \delta$.
\end{lemma}
\begin{proof}
The proof is identical to that of Lemma~\ref{lem::mle-concentration}, except that one replaces $\TT$ with $\Theta$, $T_\star$ with $O_\star$, and the sample pairs $\tilde x^\ell_h, x^\ell_h, a^\ell_h$ with $y^\ell_h, x^\ell_h$.
\end{proof}

\begin{lemma}
    $P(\EE_{\text{Fre}} ) \geq 1-  \delta$.
\end{lemma}
\begin{proof}
    Fix $T \in \TT$ and $O \in \Theta$. For convenience, define \begin{align}\epsilon(x, a) =  \| T_\star(\cdot |x, a) - T(\cdot  | x, a)  \|_1 + \| O_\star (\cdot | x ) - O(\cdot | x) \|_1.\end{align}

Then, consider the stochastic process given by 
\begin{align*}
    Z_\ell =  \epsilon(x_h^\ell, a_h^\ell) 
\end{align*}
It is easy to see that $Z_\ell -  \E_{\tilde \pi_\ell} \left[Z_\ell\right]$ is a martingale difference sequence with $|Z_\ell| \leq 8 =: R$ since $x_{h}^\ell, a_h^\ell$ are drawn from the exploration policy $\tilde \pi_\ell$.  By Theorem~1 of \cite{beygelzimer2011contextual}, with probability at least $1 - \delta$, for all $k \in [K]$,
\begin{align}
    \sum_{\ell \in [k]} \E_{\tilde \pi_\ell} \left[Z_\ell \right] - Z_\ell & \leq {1 \over 2R} \sum_{\ell \in [k]} \var_{\tilde \pi_\ell} \left(Z_\ell \right)  +  2 R \log(1/\delta) 
\end{align}
where $\var_{\tilde \pi_\ell} (Z_\ell) = \E_{\tilde \pi_\ell}  \left( Z_\ell - \E_{\tilde \pi_\ell} \left[Z_\ell\right]\right)^2$. Then, note that
\begin{align*}
     \var_{\tilde \pi_\ell} (Z_\ell) & \leq \E_{\tilde \pi_\ell} \left[ Z_\ell^2 \right] \leq R \E_{\tilde \pi_\ell} Z_\ell
\end{align*}
since $0 \leq Z_\ell \leq R$. Applying this inequality and then rearranging, we have
\begin{align*}
    \sum_{\ell \in [k]} \E_{\tilde \pi_\ell} \left[Z_\ell \right] & \leq 2 \sum_{\ell \in [k]} Z_\ell + 4R\log(1/\delta)
\end{align*}
Applying the definition of $Z_\ell$ and taking the union bound for all $h \in [H]$, $k \in [K']$, $T \in \TT$ and $O \in \Theta$ gives the result with $R = 8$.
\end{proof}

\subsection{Consequences of concentration}

\begin{lemma}Assume that the events $\EE_{\TT, \Theta}$ and $\EE_\text{Fre}$ hold. 
    For all $k \in [K']$, $T_\star \in \TT_k$ and $O_\star \in \Theta_k$.
\end{lemma}
\begin{proof}
    Note that from $\EE_{\TT, \Theta}$, it holds that for all $k \in [K']$ and $T \in \TT$ and $O \in \Theta$,
    \begin{align}
        \sum_{\ell \in [k - 1], h} \log T(\tilde x^\ell_h  | x^\ell_h, a^\ell_h)  - 2 \log (K' |\TT| /\delta) \leq   \sum_{\ell \in [k - 1], h} \log T_\star(\tilde x^\ell_h  | x^\ell_h, a^\ell_h)
    \end{align}
    and 
    \begin{align}
        \sum_{\ell \in [k - 1], h} \log O(y^\ell_h  | x^\ell_h)  - 2 \log (K' |\Theta| /\delta) \leq  \sum_{\ell \in [k - 1], h} \log O_\star(y^\ell_h  | x^\ell_h)
    \end{align}

    Given the definitions of $\beta_\TT$ and $\beta_\Theta$, the result is immediate.
\end{proof}

\begin{lemma}\label{lem::gen-approx-tv}
    Assume that the events $\EE_{\TT, \Theta}$ and $\EE_\text{Fre}$ hold. Then, for all $k \in [K']$ and $h \in [H]$, with $\tilde \pi_\ell = \hat \pi_\ell \circ_h \unif(\AA)$,
    \begin{align*}
    &  \sum_{\ell \in [k - 1]} \E_{\tilde \pi_\ell} \left[ \| T_\star(\cdot | x_h, a_h) - \hat T_k(\cdot | x_h, a_h) \|_1^2 + \| O_\star(\cdot | x_h)  - \hat O_k(\cdot  | x_h) \|_1^2 \right]   \\
    & \leq 32  \log ( K' H  | \TT \times \Theta | /\delta)  +  16 \left( \beta_\Theta + \beta_\TT \right)
    \end{align*}
\end{lemma}
\begin{proof}
    From $\EE_\text{Fre}$. Fix $h \in [H]$ and $k \in [K']$. Then,
    \begin{align}
        & \sum_{\ell} \E_{\tilde \pi_\ell} \left[ \| T_\star(\cdot | x_h, a_h) - \hat T_k(\cdot | x_h, a_h) \|_1^2 + \| O_\star(\cdot | x_h)  - \hat O_k(\cdot  | x_h) \|_1^2 \right]  \\
        & \leq 32 \log ( K'  H | \TT \times \Theta | /\delta)  +  2  \sum_{\ell}  \|  T_\star(\cdot | x_h^\ell, a_h^\ell) - \hat T_k (\cdot | x_h^\ell, a_h^\ell) \|_1^2 + \| O_\star(\cdot | x_h^\ell)  - \hat O_k(\cdot  | x_h^\ell) \|_1^2 
    \end{align}
    Then, using $\EE_{\TT,\Theta}$,
    \begin{align}
    \sum_{\ell} \| O_\star(\cdot | x_h^\ell)  - \hat O_k(\cdot  | x_h^\ell) \|_1^2    & \leq  \sum_{\ell, h}  \| O_\star(\cdot | x_h^\ell)  - \hat O_k(\cdot  | x_h^\ell) \|_1^2 \\
    & \leq  8 \log (K' | \Theta | /\delta) + 4 \sum_{\ell, h} \log {O_\star(y^\ell_h | x^\ell_h) \over \hat O_k(y^\ell_h | x^\ell_h) } \\
    & \leq 16 \log (K' | \Theta | /\delta) 
    \end{align}
    where the last inequality uses the fact that $\hat O_k \in \Theta_k$ and $\max_{O \in \Theta_k} \sum_{\ell, h}  \log O(y^\ell_h | x^\ell_h ) \geq\sum_{\ell, h}  \log O_\star(y^\ell_h | x^\ell_h )$ since $O_\star \in \Theta_k$. The same can be done for $\hat T_k$ and $\TT_k$. Then the prior display can be bounded as
    \begin{align}
        & \sum_{\ell} \E_{\tilde \pi_\ell} \left[ \| T_\star(\cdot | x_h, a_h) - T(\cdot | x_h, a_h) \|_1^2 + \| O_\star(\cdot | x_h)  - O(\cdot  | x_h) \|_1^2 \right]  \\
        & \leq 32  \log ( K' H | \TT \times \Theta | /\delta)  +  32 \left( \log (K' | \Theta  | /\delta) +  \log ( K' | \TT | /\delta) \right)  \\
        & = 32  \log ( K' H | \TT \times \Theta | /\delta)  +  16 \left( \beta_\Theta + \beta_\TT \right).
    \end{align}
\end{proof}

\subsection{Final steps}

The final steps of the proof follow a classic optimism analysis.  We let $\hat v_k$ denote the value function of the POMDP under the model transition function $\hat T_k$ and emission function $\hat O_k$ (selected optimistically in the algorithm).

The instantaneous regret for $k \in [K']$ is bounded as 
\begin{align}
    v(\pistar) - v(\hat \pi_k)  & \leq \hat v_k (\hat \pi_k) - v(\hat \pi_k) \\
    & \leq   H \sum_{h}\E_{\hat \pi_k }  \| T_\star (\cdot | x_h, a_h) - \hat T_k(\cdot | x_h, a_h) \|_1 + \| O_\star(\cdot | x_{h} ) - \hat O_k (\cdot | x_{h} ) \|_1
\end{align}
where the second line follows from the POMDP Simulation Lemma (Lemma~\ref{lem::sim-lem}). Next, leveraging the low rank MDP assumption, define the following quantities: 
\begin{align}
    \epsilon(T, O, x, a) & := \| T(\cdot | x, a) - T_\star(\cdot | x, a) \|_1 + \| O(\cdot | x) -  O_\star(\cdot | x) \|_1  \\
    U_{h}(\hat \pi_k) & := \E_{x_{h - 1}, a_{h - 1} \hat \pi_k}  \left[ \phi(x_{ h - 1}, a_{h - 1})  \right] \\
    W_h(T, O) & = \int_{x_h} \psi_\star(x_h) \cdot \sup_{a_h} \epsilon(T, O, x_h, a_h) \cdot \diff x_h \\
    \Sigma_{k, h} & := \lambda I + \sum_{\ell \in [k - 1]}  U_h(\hat \pi_\ell) U_h(\hat \pi_\ell)^\top
\end{align} 
for some $\lambda > 0$ to be determined later. For $h -1 = 0$, we can take $U_h$ to be a fixed indicator and $W_h$ to also be an indicator with a non-zero value of $\E_{x_1 \sim \rho} \sup_{a} \epsilon(T, O, x_1, a)$.
The algorithm does not need to use the vector functions $U_h$ or $W_h$ or the covariance matrix $\Sigma_{k, h}$. Only the analysis uses them.
Then, for a fixed $h \in [H]$, let $\tilde \pi_\ell = \hat \pi_\ell \circ_h \unif(\AA)$ be the exploration policy used in round $\ell$ for timestep $h$. Then, letting $\tau_h$ denote the concatenation of  $(\tau_{h - 1}, a_{h - 1}, y_h)$ as usual,
\begin{align}
    & \E_{\hat \pi_k }  \| T_\star (\cdot | x_h, a_h) - \hat T_k(\cdot | x_h, a_h) \|_1 + \| O_\star(\cdot | x_{h} ) - \hat O_k (\cdot | x_{h} ) \|_1  \\
    & = \E_{x_{h - 1}, a_{h - 1}, \tau_{h - 1} \sim \hat \pi_k } \< \phi_\star ( x_{h - 1}, a_{h - 1} ), \int_{ y_h, x_h, a_h} \psi_\star (x_h) O_\star(y_h | x_h ) \hat \pi_k(a_h | \tau_h) \epsilon(\hat T_k, \hat O_k, x_h, a_h) \cdot \diff (y_h, x_h, a_h) \>  \\
    & \leq \E_{x_{h - 1}, a_{h - 1} \sim \hat \pi_k }  \< \phi_\star ( x_{h - 1}, a_{h - 1} ), \int_{ y_h, x_h} \psi_\star (x_h) O_\star(y_h | x_h )   \sup_{a_h} \epsilon(\hat T_k, \hat O_k, x_h, a_h) \cdot \diff (y_h, x_h) \>  \\
    & \leq 
    \| U_h(\hat \pi_k)  \|_{\Sigma_{k, h}^{-1}  } \cdot  \left\lVert \int_{ x_h} \psi_\star (x_h)   \sup_{a_h} \epsilon(\hat T_k, \hat O_k, x_h, a_h) \cdot \diff  x_h \right\rVert_{\Sigma_{k, h}} \\
    & = \| U_h(\hat \pi_k)  \|_{\Sigma_{k, h}^{-1}  } \cdot  \| W_h(\hat T_k, \hat O_k) \|_{\Sigma_{k, h}}.
    \end{align}

The first line uses the definition of the low-rank latent transition to decompose the expectation over elements at step $h$. The first inequality replaces the distribution over $a_h$ with a $\sup$, which is valid because $\phi_{\star}(x_{h - 1}, a_{h - 1})^\top \psi_{\star}(x_h) \geq 0$ is a probability. The second inequality applies the Cauchy-Schwarz inequality with the definition of $U_h$ and the last line uses the definition of $W_h$.
For the right-hand factor,
\begin{align}
     \lVert W_h(\hat T_k, \hat O_k ) \rVert_{\Sigma_{k, h}}^2  & = \lambda \| W_h(\hat T_k, \hat O_k) \|_2^2 +  \sum_{\ell \in [k - 1] }  \< U_h(\hat \pi_{\ell}), 
     W_h (\hat T_k, \hat O_k) \>^2 
    \end{align}
    Using the normalization condition on $\psi_\star$, we have
    \begin{align}
         \lambda \| W_h(\hat T_k, \hat O_k) \|_2^2 & \leq 8 \lambda d  
    \end{align}
    Also,
    \begin{align}
        &  \sum_{\ell \in [k - 1] }  \< U_h(\hat \pi_{\ell}), 
     W_h (\hat T_k, \hat O_k) \>^2  \\
     & = \sum_{\ell \in [k - 1]} \left( \E_{x_h \sim \hat \pi_\ell } \sup_{a_h} \epsilon(\hat T_k, \hat O_k, x_h, a_h) \right)^2  \\
     & \leq 2 \sum_{\ell \in [k - 1]} \E_{x_h \sim \hat \pi_\ell} \sup_{a_h} \left[ \| T_\star(\cdot | x_h, a_h) - \hat T_k(\cdot | x_h, a_h)   \|_1^2 + \| O_\star(\cdot | x_h) - \hat O_k(\cdot | x_h) \|_1^2 \right]  \\
     & \leq 2 A \sum_{\ell \in [k - 1 ]} \E_{x_h, a_h \sim \tilde \pi_\ell} \left[ \| T_\star(\cdot | x_h, a_h) - \hat T_k(\cdot | x_h, a_h)   \|_1^2 + \| O_\star(\cdot | x_h) - \hat O_k(\cdot | x_h) \|_1^2 \right].
    \end{align}
    The first inequality above applies Jensen's inequality and the fact that $(a  +b)^2 \leq 2(a^2 + b^2)$. The second inequality upper bounds the $\sup_{a_h}$ with a sum to convert the expression to a uniform distribution over $a_h$, which is exactly the distribution under the exploration policy $\tilde \pi_\ell$.
Therefore,  leveraging Lemma~\ref{lem::gen-approx-tv} under $\EE_\Theta$, $\EE_\TT$, and $\EE_{\text{Fre}}$, 
\begin{align}
    \lVert W_h(\hat T_k, \hat O_k ) \rVert_{\Sigma_{k, h}}^2 
     & \leq  8\lambda d + 64 A  \left( \log \left( K'H |\TT \times \Theta | /\delta \right) + \beta_\TT + \beta_\Theta  \right)  
\end{align}

Then,
\begin{align}
    & \E_{\hat \pi_k }  \| T_\star (\cdot | x_h, a_h) - \hat T_k(\cdot | x_h, a_h) \|_1 + \| O_\star(\cdot | x_{h} ) - \hat O_k (\cdot | x_{h} ) \|_1  \\
    & \leq  \min \left\{ 4,    {\sqrt{64 A} }  \lVert U_h(\hat \pi_k) \rVert_{\Sigma^{-1}_{k, h}} \sqrt{  \lambda d + \beta_\TT + \beta_\Theta +  \log( K' H | \TT \times \Theta | /\delta) }  \right\}
\end{align}
Let $\beta_\lambda := {  \lambda d + \beta_\TT + \beta_\Theta + \log( K'  H | \TT \times \Theta | /\delta) }$ for shorthand.
Then, the total sub-optimality of the proposed policies $\hat \pi_1, \ldots, \hat \pi_K$ is bounded as
\begin{align}
     \sum_{k \in [K'] } v(\pistar) - v(\hat \pi_k)
    & \leq H \sum_{k, h}  \left( 4 \wedge  \lVert U_h(\hat \pi_k) \rVert_{\Sigma^{-1}_{k, h}} \cdot \sqrt{ 64 \beta_\lambda A } \right) \\
     & \leq  {  H} \sum_{h} \sqrt{ 64 \beta_\lambda A   K  \sum_k  \left( 1 \wedge  \lVert U_h(\hat \pi_k)  \rVert^2_{\Sigma^{-1}_{k, h}} \right)    }  \\
     & \leq  {H^2   } \sqrt{  64  \beta_{1/d}  A K d \log (1 + K')  }
\end{align}
where the last inequality applies the elliptical potential lemma (Lemma~\ref{lem::potential}) with the setting $\lambda = \nicefrac{1}{d}$.

\subsection{Proof of Lemma~\ref{lem::mle-concentration}} \label{app::mle-concentration-proof}

\begin{proof} The proof follows a similar approach as \citet{agarwal2020flambe}. 
    Fix $k \in [K']$. 
    In contrast to the the rest of the paper, in this proof only we use $\bar\tau^\ell$ to denote the data collected, since these differ from the histories in the usual sense as a result of the exploration happening over $H$ rounds per epoch. In particular, we define $\bar \tau^\ell =  \left\{ x_h^\ell, a_h^\ell, y_h^\ell, \tilde x_h^\ell  \right\}_{h = \in [H]}$ where we recall in the algorithm that $\tilde x_{h}^\ell$ is the next-state sampled by the exploration policy for step $h$.
    
    Given the full trajectories $\bar \tau^1, \ldots, \bar\tau^k$, define the tangent dataset sampled independently as  $\hat x^\ell_{h + 1} \sim T_\star (\cdot  | x^\ell_h, a^\ell_h)$ for $\ell \in [k]$, which has a distribution completely determined by the history. Let $l: \XX^2 \times \AA  \to \R$ be an arbitrary measurable loss function. Define $L = \sum_{\ell, h} l(x^\ell_h, a^\ell_{h}, \tilde x^\ell_{h} )$ and $\hat L = \sum_{\ell, h} l( x^\ell_h, a^\ell_{h}, \hat x^\ell_{h})$.
    For convenience, let us define $\bar{ \taub}^\ell_h$ to be the concatenated sequence $(\bar \tau^{1:\ell- 1}, \bar \tau^{\ell}_h, a_h^\ell)$.

	Consider the function:
	\begin{align*}
		\exp \left(  L -\log \E\left[  \exp (\hat L)  \ | \ \bar{ \taub}^k_H \right]  \right) & = {\exp (L) \over  \E \left[  \exp (\hat L) \ | \  \bar{ \taub}^k_H \right] } \\
		& = {\exp (L) \over  \prod_{\ell, h}  \E \left[  \exp l( x^\ell_h,a^\ell_h, \hat x^\ell_{h} )\  | \ \bar{\taub}^\ell_{h} 
		\right] }
	\end{align*}
	where the second equality follows from the fact that the tangent observations are independent given the history of latent states. Then,
	\begin{align}
		\E \left[  \exp \left(  L -\log \E\left[  \exp (\hat L)  \ | \ \bar{ \taub}^k_h  \right]  \right) \right] & = \E \left[  {\exp (L) \over  \prod_{\ell, h}  \E \left[  \exp l(x^\ell_h,a^\ell_h, \hat x^\ell_{h}) \ | \ \bar{ \taub}^\ell_h \right] }  \right] \\
		& = \E \left[  { \prod_{\ell, h} \exp  l(x^\ell_h,a^\ell_h,  \tilde x^\ell_{h})   \over  \prod_{\ell, h}  \E \left[  \exp l( x^\ell_h,a^\ell_h, \hat x^\ell_{h}) \ | \  \bar{\taub}^\ell_h \right] }   \right] \\
		& = \E \left[  { \prod_{\ell \in [k], h \in [H - 1]} \exp \left(  l(x^\ell_h,a^\ell_h, \tilde x^\ell_{h}) \right)  \over  \prod_{\ell \in [k], h \in [H - 1]}  \E \left[  \exp l(x^\ell_h,a^\ell_h, \hat x^\ell_{h + 1}) \ | \  \bar{\taub}^\ell_h \right] }  \cdot 
		{  \E \left[  \exp\left( l(x^k_H,a^k_H,  \tilde x^k_{H} )  \right)  \ | \ \bar \taub^k_{H} \right]  \over  \E \left[  \exp\left( l(x^k_H,a^k_H, \hat x^k_{H} )  \right)  \ | \ \bar{\taub}^k_{H} \right]  }  \right]  \\
		& = \E \left[  { \prod_{\ell \in [k], h \in [H - 1]} \exp \left(  l(x^\ell_h,a^\ell_h,  \tilde x^\ell_{h + 1}) \right)  \over  \prod_{\ell \in [k], h \in [H - 1]}  \E \left[  \exp l( x^\ell_h,a^\ell_h, \hat x^\ell_{h + 1}) \ | \  \bar \taub^\ell_{h} \right] }  \right] \\
		& = \ldots \\
		& = 1
	\end{align}
	where the cancellation is repeated for all $\ell \in [k]$ and $h \in [H]$. Applying Markov's inequality ensures that
	\begin{align}
		P \left( L -\log \E\left[  \exp (\hat L)  \ | \ \bar \taub^k_H \right] \geq z   \right) & \leq {\E \left[  \exp \left( L -\log \E\left[  \exp (\hat L)  \ | \ \bar \taub^k_H \right]  \right)  \right] \over e^z }.
	\end{align}
	Taking $z = \log(1/\delta)$ guarantees that
	\begin{align}
		L -\log \E\left[  \exp (\hat L)  \ | \ \bar\taub^k_H \right] \leq \log(1/\delta)
	\end{align}
	with probability at least $1 - \delta$.
	
	We will define the loss function as $\ell(x, a, x') = \log \sqrt{ T(x' | x, a) \over T_\star(x' | x, a) }$ for some $T \in \TT$.
	One can then relate the above loss function to the total variation distance:
	\begin{align}
		\sum_{\ell \in [k], h \in [H]}  \| T_\star (\cdot | x^\ell_h, a^\ell_h) - T(\cdot |x^\ell_h, a^\ell_h ) \|^2_1 & = \sum_{\ell, h}  \left( \int_{x' \in \XX} | T_\star (x' | x^\ell_h, a_{h}^\ell) - T(x'  |x^\ell_h, a_{h}^\ell) | \diff x' \right)^2 \\
		& \leq 4 \sum_{\ell, h}  \int_{x'} \left( \sqrt{T_\star (x' | x^\ell_h, a^\ell_h)} - \sqrt{T( x'  |x^\ell_h,a^\ell_h)} \right)^2 \diff x'    \\
		& = 8 \sum_{\ell, h}  \E_{x' \sim T_\star(\cdot | x^\ell_h, a^\ell_h) }  \left [ 1 -   \sqrt{ T(x' | x^\ell_h,a^\ell_h)/T_\star (x'  | x^\ell_h, a^\ell_h) } \right]     \\
		& \leq  - 8 \sum_{\ell, h} \log \E_{x' \sim T_\star (\cdot | x^\ell_h, a^\ell_h)} \sqrt{ T(x'  | x^\ell_h, a^\ell_h) \over  T_{\star} (x' | x^\ell_h, a^\ell_h) }
	\end{align}
	where the first inequality is from Cauchy-Schwarz and the second uses the fact that $\log (1 + a) \leq a$ for $a > -1$. Observe that we have
	\begin{align}
		- \log \E \left[  \exp(\hat L) \ | \ \bar \taub^k_H \right] 
		& = - \log \E  \left[  \prod_{\ell, h} \exp \left( \log \left( \sqrt{ T(\hat x^\ell_{h + 1}  | x^\ell_h, a^\ell_h) \over T_\star ( \hat  x^\ell_{h + 1}  | x^\ell_h, a^\ell_h) }  \right)    \right)  \ | \ \bar \taub^k_H \right] \\
		& = - \sum_{\ell, h} \log  \left(  \E_{x' \sim T_\star (\cdot | x^\ell_h, a^\ell_h) }  \sqrt{ T(x' | x^\ell_h, a^\ell_h) \over T_\star ( x' | x^\ell_h, a^\ell_h)  }  \right)
	\end{align}
	Combining this with the concentration inequality from earlier, we conclude that
	\begin{align}
		\sum_{\ell, h}  \| T_\star(\cdot | x^\ell_h, a^\ell_h ) - T(\cdot |x^\ell_h, a^\ell_h) \|^2_1 & \leq - 8 \log \E \left[  \exp(\hat L) \ | \ \x^k_H \right] \\
		& \leq - 8 \left( L  - \log(1/\delta) \right) \\
		& = -8  \left( {1 \over 2} \sum_{\ell, h }  \log \left({ T(x^\ell_{h + 1} | x^\ell_h,a^\ell_h) \over T_\star (x^\ell_{h + 1} | x^\ell_h,a^\ell_h)} \right) - \log(1/\delta)  \right)
	\end{align}
	with probability at least $1 - \delta$ for a fixed $T$. Taking the union bound over all $T \in \TT$ and all $k \in [K']$, 
	
	\begin{align}
	    \sum_{\ell \in [k], h \in [H]} \| T_\star(\cdot |x^\ell_h, a^\ell_h) - T(\cdot | x^\ell_h, a^\ell_h) \|_1^2 & \leq 8 \log (K' |\TT| /\delta) + 4 \sum_{\ell, h} \log \frac{ T_\star(x^\ell_{h+ 1} | x^\ell_h, a^\ell_h ) }{ T(x^\ell_{h + 1}, x^\ell_h, a^\ell_h)  }
	\end{align}
	
	 with probability at least $1 - \delta$. 
\end{proof}

\section{Auxiliary Lemmas}

Here state and prove a number of helpful auxiliary results for the main theorems.

\subsection{Simulation lemma}

\simulation*

\begin{proof}
Note that
\begin{align*}
    |v(\pi) - \hat v(\pi) | & \leq  H \| P_\pi - \hat P_\pi \|_1
\end{align*}
where $P_\pi$ denotes the measure over trajectories $\bar\tau = (x_1, y_1, a_1, \ldots, x_H, y_H, a_H)$ including the latent states under the true model with $T$ and $O$ and policy $\pi$. Similarly $\hat P_\pi$ denotes the same measure but under the model with $\hat T$ and $\hat O$. With notation slightly abused, we also use $P_\pi(\bar \tau)$ to denote the density. This density decomposes as 
	\begin{align*}
		P_\pi(\bar \tau) = \rho (x_1) \left(  \prod_{h \in [H - 1] } T_\star(x_{h + 1} | x_h, a_h)  O_\star(y_h | x_h) \pi(a_h |\tau_h) \right)  O_\star(y_H | x_H) \pi(a_H | \tau_H),
	\end{align*}
	where we recall that $\tau_h = (y_{1:h}, a_{1:h - 1})$ is the partial history. $\hat P_\pi$ is analogously defined.
	Consider for now a fixed $\bar \tau$. To bound the total variation distance, we are interested in bounding the differences between $P_\pi(\bar\tau)$ and $\hat P_\pi(\bar \tau)$. For shorthand, we will define the following quantities:
	\begin{align*}
		B_{h} & =  \rho(x_1) O_\star(y_1 | x_1)  \prod_{t \in [h - 1] } T_\star(x_{t + 1} | x_t, a_t) O_\star (y_{t + 1}  | x_{t + 1}),  \\
		\bar \pi_h &  = \prod_{t \in [h]} \pi(a_t | \tau_{t}).
	\end{align*}
	Note that we have the following recursion:
	\begin{align}
	B_{h} = O_\star(y_h  | x_h) T_\star(x_{h} | x_{ h- 1} , a_{h - 1} ) B_{h - 1}
	\end{align}
	where $B_1 = \rho(x_1) O_\star(y_1 | x_1)$.
	We also define $\hat B_h$ analogously with $\hat O$ and $\hat T$.
	To prove Lemma~\ref{lem::sim-lem}, we will recursively apply the following bound.
	
	\begin{restatable}{lemma}{lemrecursion}
		\label{lem::recursion}
		The following inequality holds:
		\begin{align}
		\int_{x_h, y_h, a_h} \bar \pi_h |B_h - \hat B_h | \cdot \diff(x_h, y_h, a_h)  & \leq \bar \pi_{h - 1} \int_{x_h} T_\star(x_h | x_{h -1 }, a_{ h- 1} ) B_{ h- 1} \cdot \| O_\star(\cdot | x_h ) - \hat O(\cdot |x_h) \|_1  \cdot \diff x_h\\
		& \quad + \bar \pi_{ h- 1} B_{ h -1} \cdot  \| T_\star(\cdot | x_{ h- 1}, a_{ h- 1}) - \hat T(\cdot | x_{ h- 1}, a_{ h- 1} ) \|_1  \\
		&\quad + \bar \pi_{ h- 1} | B_{ h- 1} - \hat B_{ h- 1}|
		\end{align}
	\end{restatable}

Summing over all possible latent-augmented trajectories, this implies that we have
\begin{align}
\int_{\bar \tau} | P_\pi(\bar \tau) - P_\pi(\bar \tau)| \cdot \diff \bar \tau & = \int_{\bar \tau}  \bar \pi_H| B_H - \hat B_H | \cdot \diff \bar \tau \\
& \leq \sum_{h}  \int_{x_h} P_\pi(x_h) \cdot \| O_\star(\cdot | x_h) - \hat O(\cdot | x_h) \|_1  \cdot \diff x_h \\
&  \quad + \sum_{h \in [ H - 1] }  \int_{x_h, a_h} P_\pi(x_h, a_h)\cdot \| T_\star(\cdot |x_{h}, a_h) - \hat T(\cdot | x_{h}, a_{h}) \|_1 \\
& = \E_\pi \sum_{h}  \| O_\star(\cdot | x_h) - \hat O(\cdot | x_h) \|_1  + \E_{\pi}\sum_{h \in [H - 1] } \| T_\star(\cdot |x_{h}, a_h) - \hat T(\cdot | x_{h}, a_{h}) \|_1.
\end{align}

\end{proof}

\subsubsection{Proof of Lemma~\ref{lem::recursion}}

\lemrecursion*

	\begin{proof}
	We first average out $a_h$ and then apply the triangle inequality to bound the quantity in terms of the difference in emission matrices $\| O(\cdot | x_h) - \hat O(\cdot | x_h) \|_1$:
	\begin{align}
	& \int_{x_h, y_h, a_h} \bar \pi_h |B_h - \hat B_h | \cdot \diff (x_h, y_h, a_h) \\
	& = \bar \pi_{h - 1} \int_{x_h, y_h}  |B_h - \hat B_h | \cdot \diff (x_h, y_h) \\
	& = \bar \pi_{h - 1} \int_{x_h, y_h}  | T_\star(x_h | x_{h - 1}, a_{ h- 1}) O_\star(y_{h} | x_h )B_{h - 1} - \hat T(x_h | x_{h - 1}, a_{ h- 1}) \hat O(y_{h} | x_h )\hat B_{h - 1} |\cdot \diff (x_h, y_h) \\
	& \leq \bar \pi_{h - 1} \int_{x_h, y_h} T_\star(x_h | x_{h - 1}, a_{h - 1}) B_{ h- 1}  \cdot | O_\star(y_h | x_h) - \hat O(y_h | x_h) | \cdot \diff (x_h, y_h) \\
	& \quad + \bar \pi_{h-1 } \int_{x_h, y_h } \hat O(y_h | x_h) | T_\star(x_{h} | x_{h - 1}, a_{ h -1 } ) B_{ h- 1} - \hat T(x_{h} | x_{h - 1}, a_{ h -1 } ) \hat B_{ h- 1} | \cdot \diff (x_h, y_h) \\
	& =   \underbrace{\bar \pi_{h - 1} \int_{x_h} T_\star(x_h | x_{h - 1}, a_{h - 1}) B_{ h- 1} \cdot \| O_\star(\cdot  | x_h) - \hat O(\cdot | x_h) \|_1\cdot \diff (x_h)}_{\textbf{(I)}} \\
	& \quad +  \bar \pi_{ h- 1} \int_{x_h} | T_\star(x_{h} | x_{h - 1}, a_{ h -1 } ) B_{ h- 1} - \hat T(x_{h} | x_{h - 1}, a_{ h -1 } ) \hat B_{ h- 1} | \cdot \diff (x_h).
	\end{align}
	Now, we can also apply the triangle inequality to the last term on the right side to bound the quantity in terms of the difference in transition matrices:
	\begin{align}
\int_{x_h, y_h, a_h} \bar \pi_h |B_h - \hat B_h | \cdot \diff (x_h, y_h, a_h)
& \leq \textbf{(I)} + \bar \pi_{h - 1} B_{ h- 1} \int_{x_h} | T_\star(x_h | x_{ h- 1}, a_{ h- 1}) - \hat T(x_{ h} | x_{ h- 1}, a_{ h - 1}) | \cdot \diff (x_h)  \\
& \quad + \bar \pi_{h- 1}  \int_{ x_h} \hat T(x_h | x_{ h - 1}, a_{ h - 1} ) | B_{ h- 1} - \hat B_{ h - 1} | \cdot \diff (x_h) \\
& \leq \textbf{(I)} + \bar \pi_{h - 1} B_{ h- 1} \cdot \| T(\cdot | x_{ h- 1}, a_{ h- 1}) - \hat T(\cdot | x_{ h- 1}, a_{ h- 1} ) \|_1 \\
& \quad + \bar \pi_{h- 1} | B_{ h- 1} - \hat B_{ h- 1} | .
	\end{align}
	This concludes the proof.
	\end{proof}

\subsection{Concentration inequalities}

\begin{lemma}[Hoeffding's inequality]\label{lem::hoeffding}
Let $Z_1, \ldots, Z_n$ be a sequence of independent random variables with $Z_i \in [a, b]$ for all $i$ for $-\infty <a \leq b < \infty $. Then
\begin{align*}
    P\left( {1 \over n} \sum_i Z_i - \E[Z_i] \geq (b - a) \sqrt{ \log (1/\delta) \over n }  \right) \leq \delta  .
\end{align*}
\end{lemma}

\begin{lemma}[Bernstein's inequality]\label{lem::bernstein}

Let $Z_1, \ldots, Z_n$ be a sequence of independent random variables with $Z_i \in [0, 1]$ and variance $\var(Z_i) = \sigma^2$ and mean $\E[Z_i] = \mu$ for all $i$. Then, with probability at least $1 - \delta$,
\begin{align*}
    {1 \over n} \sum_i Z_i - \E[Z_i] \leq {\log(1/\delta)  \over 3n}  +    \sqrt{ 2\sigma^2 \log (1/\delta) \over n }   .
\end{align*}
Furthermore, this implies that, for all $c \geq 1$,
\begin{align*}
    {1 \over n} \sum_i Z_i - \mu & \leq {\log(1/\delta)  \over 3n}  +    \sqrt{ 2\mu \log (1/\delta) \over n } \\
    & \leq {\mu \over 2c } + {2 c\log(1/\delta) \over n } .
\end{align*}
\end{lemma}
\begin{proof}
The first statement is simply the original statement of Bernstein's inequality. The second uses the fact that $\sigma^2 \leq \mu$ for variables in $[0, 1]$. The last one uses the AM-GM inequality.
\end{proof}

\begin{lemma}[Azuma-Hoeffding]\label{lem::azuma}
	 Let $Z_1, \ldots, Z_n$ be a martingale difference sequence with $|Z_i| \leq G$ for all $i$. Then, with probability at least $1 - \delta$,
	\begin{align}
	 \sum_{i}  Z_i  \leq  4G \sqrt{ n \log(1/\delta) }.
	\end{align}
\end{lemma}

\subsection{Pigeonhole lemmas}

\begin{lemma}[Pigeonhole Principle] \label{lem::pigeon}
	The following inequalities hold:
	\begin{align}
\sum_{k \in [K], h \in [H] }  \sqrt{1 \over \max\{ 1, n_k(x_h^k) \} } \leq  HX +  3 \sqrt{ HX K   }  
	\end{align}
	and
	\begin{align}
\sum_{k \in [K], h \in [H] }  \sqrt{1 \over \max\{ 1, n_k(x_h^k, a_h^k) \} } \leq  HXA +  3 \sqrt{ HXA K   }.
	\end{align}

\end{lemma}
\begin{proof}
	We prove only the first as the second is equivalent up to summations over the actions. Note that 
	\begin{align}
	\sum_{k \in [K], h \in [H] }  \sqrt{1 \over \max\{ 1, n_k(x_h^k) \}} & = \sum_{x}  \sum_{k = 1}^K  { m_k(x) \over  \sqrt{ \max\{ 1, n_k(x) \} } }  \\
	& \leq X H + \sum_{x}  \sum_{k = 1}^K  { m_k(x) \over  \sqrt{ \max\{ H, n_k(x) \} } },
	\end{align}where $m_k(x) = \sum_{h \in [H]} \1 \{  x^k_h = x \}$ counts the number of occurrences of $x$ in a single round $k$. The inequality uses the fact that, for any $x$, the summand can contribute at most $H$ to the sum (because $m_k(x)$ is bounded by $H$) before $n_k(x)$ has value at least $H$. Now we use Lemma~\ref{lem::pigeon-supporting} (which is adapted from Lemma~19 of \cite{auer2008near}) to bound the second term:
	\begin{align}
	\sum_{k \in [K], h \in [H] }  \sqrt{1 \over \max\{ H, n_k(x_h^k) \}} & \leq X H +  3\sum_x \sqrt{ n_K(x) } \\
	& \leq X H + 3 \sqrt{ HX K }, 
	\end{align}
	where the last line follows from the Cauchy-Schwarz inequality along with the fact that $\sum_{x} n_K(x) = KH$.

\end{proof}	
	
	\begin{lemma}[Adapted from \citet{auer2008near}]\label{lem::pigeon-supporting}
		Let $z_1, \ldots, z_n \in [0, H ]$ be an arbitrary sequence and let $Z_k = \max\{ H, \sum_{k= 1}^{k} z_k \}$. Then,\begin{align}
		\sum_{k \in [n] } { z_k \over \sqrt{ Z_{k - 1} } } \leq 3 \sqrt{Z_n}.
		\end{align}
	\end{lemma}

\begin{proof}
	Consider the case where $n = 1$. Then, $Z_0 = H$. Furthermore,
	\begin{align}
	\sum_{k \in [n] } {z_k \over \sqrt{ Z_k - 1} } & =  { z_1 \over \sqrt{H}   }  \leq \sqrt{H} \leq 3 \sqrt{Z_1}
	\end{align}
	
	By induction on the base case, we have
	\begin{align}
	\sum_{k \in [n] } { z_k \over \sqrt{ Z_{k - 1} } } & = 3 \sqrt{Z_{n - 1}} +  { z_n \over \sqrt{Z_{n - 1}}  }
	\end{align}
	\begin{align}
	\left(3 \sqrt{Z_{n - 1}} +  { z_n \over \sqrt{Z_{n - 1}}  } \right)^2 & = 9 Z_{n - 1} + 6z_n + { z_n^2 \over Z_{n - 1}}  \\
	& \leq 9 Z_{n - 1} + 7 z_n \\
	& \leq 9 Z_{n}   
	\end{align}
	Therefore, by taking the square root,
	\begin{align}
	\sum_{k \in [n] } { z_k \over \sqrt{ Z_{k - 1} } } & \leq 3 \sqrt{ Z_{n} } 
	\end{align}
\end{proof}

\begin{lemma}\label{lem::pigeon2}
	The following inequality holds:
	\begin{align}
	 \sum_{k \in [K], h\in [H]} {1 \over \max \left\{   1, n_k(x_h^k, a_h^k) \right\}} \leq HXA (1 + \log K ).
	\end{align}
	
	\begin{proof}
		First note that since $n_k(x, a)$ is updated each episode, we immediately have
		\begin{align}
		 \sum_{k \in [K], h\in [H]} {1 \over \max \left\{   1, n_k(x_h^k) \right\}} & \leq   HX A \sum_{i =1}^{ \ceil{K/ XA} } {1 \over i} \\
		 &  \leq HX A \sum_{i =1}^{ K } {1 \over i}
		\end{align}
		since we assume that $X, A \geq 1$.  Then,
		\begin{align}
		 \sum_{i \in [K]} \frac{1}{i} \leq  1+ \int_1^K \frac{dx}{x}   \\
		 = 1 + \log K .
		\end{align}
	\end{proof}
\end{lemma}

\begin{lemma}[\citet{lattimore2020bandit}]
    \label{lem::potential}
    Let $\Sigma_{k} = \lambda I  + \sum_{\ell \in [k - 1] }  \phi_{\ell} \phi_{\ell}^\top$ and $\| \phi_\ell \| \leq 1$ uniformly. Then, 
    \begin{align}
    \sum_{k \in [K]}  \left( 1 \wedge \| \phi_k\|^2_{\Sigma_k^{-1}} \right) \leq 2 d \log \left( \frac{d \lambda  + K }{d \lambda } \right) .
    \end{align}
    
\end{lemma}

\end{document}